\documentclass[onefignum,onetabnum]{siamonline181217}

\usepackage{amsfonts,amsmath,enumitem}
\usepackage{graphicx}
\usepackage{epstopdf}
\usepackage{algorithmic}

\usepackage{amsmath,amssymb}
\usepackage{color}
\usepackage[english]{babel}
\usepackage{mathrsfs,dsfont}
\usepackage{array}
\newcolumntype{P}[1]{>{\centering\arraybackslash}p{#1}}

\usepackage{tikz}
\usetikzlibrary{calc}
\tikzset{
  basic box/.style = {
    shape = rectangle,
    align = center,
    draw  = #1,
    fill  = #1!25,
    rounded corners},
  header node/.style = {
    Minimum Width = header nodes,
    font          = \strut\Large\ttfamily,
    text depth    = +0pt,
    fill          = white,
    draw},
  header/.style = {%
    inner ysep = +1.5em,
    append after command = {
      \pgfextra{\let\TikZlastnode\tikzlastnode}
      node [header node] (header-\TikZlastnode) at (\TikZlastnode.north) {#1}
      node [span = (\TikZlastnode)(header-\TikZlastnode)]
        at (fit bounding box) (h-\TikZlastnode) {}
    }
  }
  
  }

\usepackage{mathrsfs}
\usepackage[mathscr]{euscript}
\usepackage{tikz}
\usetikzlibrary{shapes}
\usepackage{amsopn}

\newsiamremark{rem}{Remark}
\newsiamremark{defn}{Definition}
\newsiamremark{assumption}{Assumption}
\crefname{hypothesis}{Hypothesis}{Hypotheses}
\newsiamthm{claim}{Claim}
\newsiamthm{lem}{Lemma}
\newsiamthm{cor}{Corollary}
\newsiamremark{example}{Example}
\newsiamthm{question}{Question}
\newsiamremark{open}{Open Problem}

\def\QED{~\rule[-1pt]{5pt}{5pt}\par\medskip}
\renewenvironment{proof}{{\bf Proof: \ }}{ \hfill \QED}

\let\ALP  \mathcal
\let\FLD  \mathscr

\newcommand{\ind}[1]{\mathds{1}_{\{#1\}}}
\newcommand{\beq}[1]{\begin{align} #1 \end{align}}
\newcommand{\beqq}[1]{\begin{align*} #1 \end{align*}}
\renewcommand{\Re}{\mathbb{R}}
\newcommand{\Na}{\mathbb{N}}

\newcommand{\ex}[1]{\mathds{E}\left[#1\right]}
\newcommand{\pr}[1]{\mathds{P}\left\{#1\right\}}

\renewcommand{\forall}{\text{ for all }}

\headers{Recursive Stochastic Algorithms}{A. Gupta, W. Haskell}

\title{Convergence of Recursive Stochastic Algorithms using Wasserstein Divergence
\thanks{Submitted to the editors.
\funding{Abhishek Gupta gratefully acknowledges support from NSF ECCS Award 1610615 and ARPA-E NEXTCAR Award.}}}

\author{Abhishek Gupta 
\thanks{Electrical and Computer Engineering, The Ohio State University, Columbus, OH, USA.  (\email{gupta.706@osu.edu}).}
\and William B. Haskell
\thanks{Supply Chain and Operations Management, Krannert School of Management, Purdue University, West Lafayette, IN, USA.  (\email{whaskell@purdue.edu}).}
}

\usepackage{enumitem,multirow}
\newcolumntype{P}[1]{>{\arraybackslash}m{#1}}
\newcolumntype{M}[1]{>{\centering\arraybackslash}m{#1}}

\newcommand{\new}[1]{{\color{black} #1}}
\begin{document}

\maketitle

\begin{abstract}
This paper develops a unified framework, based on iterated random operator theory, to analyze the convergence of constant stepsize recursive stochastic algorithms (RSAs). RSAs use randomization to efficiently compute expectations, and so their iterates form a stochastic process. The key idea of our analysis is to lift the RSA into an appropriate higher-dimensional space and then express it as an equivalent Markov chain. Instead of determining the convergence of this Markov chain (which may not converge under constant stepsize), we study the convergence of the distribution of this Markov chain. To study this, we define a new notion of Wasserstein divergence. We show that if the distribution of the iterates in the Markov chain satisfy a contraction property with respect to the Wasserstein divergence, then the Markov chain admits an invariant distribution. We show that convergence of a large family of constant stepsize RSAs can be understood using this framework, and we provide several detailed examples.
\end{abstract}



\section{Introduction}\label{sec:intro}
Over the past two decades, there has been explosive growth in new randomized algorithms for doing complex optimization tasks in machine and reinforcement learning. Many of these algorithms are essentially recursions of certain mappings that depend on stochastic parameters. Such algorithms are collectively called `recursive stochastic algorithms' (RSAs) \cite{ljung1977analysis,gelfand1991recursive,kushner2003stochastic}.

From a computational viewpoint, RSAs with constant stepsize (also called the `learning rate' in some contexts) enjoy many benefits compared to RSAs with decaying stepsizes that converge to zero \cite{benveniste2012adaptive}. Constant stepsize RSAs often converge much faster to a \textit{neighborhood} of the desired solution. This phenomenon has been observed in off-policy temporal difference learning \cite{yu2016weak}, temporal difference learning with function approximation \cite{lakshminarayanan2018linear}, tracking problems \cite{kumar2018bounds}, and gradient descent \cite{bach2013non}, among others. Furthermore, the size of this neighborhood is usually small if the stepsize is small (so too large a stepsize may not be beneficial) \cite{borkar2000ode,benaim1998recursive}. Accordingly, in practice researchers often use a constant stepsize for a certain number of steps, and then if needed rerun the algorithm with a smaller (constant) stepsize. 

The practical success of constant stepsize RSAs compared to decaying stepsize RSAs has spurred significant interest in finding supporting theory. The existing convergence analyses are tailored to each specific algorithm and do not readily extend to other algorithms. In this paper, we unify these analyses by highlighting the features common to all of them. We find that: (i) constant stepsize RSAs can be modeled as Markov chains and (ii) their convergence is connected to a form of contraction with respect to a new notion of divergence on probability distributions. We call this new notion the {\it Wasserstein divergence} since it is based on the classical Wasserstein distance.

We show that many constant stepsize RSAs used in optimization are contractions with respect to the Wasserstein divergence. This perspective gives us new insight into the nature of the convergence of these algorithms. In particular, the central new notion of the Wasserstein divergence serves as a unifying feature for the convergence analysis of different families of RSAs.

\subsection{Prior Work}

The study of RSAs enjoys a rich history. Early work on RSAs was for solving regression problems, where certain stepsize parameters converge to zero, see \cite{robbins1951stochastic,kiefer1952stochastic,wolfowitz1952stochastic}. These algorithms and their convergence are studied under the umbrella of stochastic approximation theory \cite{kushner2003stochastic,borkar2009stochastic}. It was soon realized that stochastic approximation theory can ascertain convergence of a wide variety of optimization and learning algorithms. To understand this method, let us consider the iteration $x_{k+1} = x_k+\beta_kd_k$, where $\{\beta_k\}$ are stepsizes that satisfy $\sum_k\beta_k = \infty$ and $\sum_k \beta_k^2 <\infty$, $d_k$ is an unbiased noisy estimate of some operator $F$ (e.g. the gradient operator) evaluated at $x_k$, and the desired solution $x^*$ satisfies $F(x^*) = 0$. The convergence guarantee for such algorithms is very strong under some reasonable conditions that are typically met in practice. However, despite strong convergence guarantees, the rate of convergence is very slow for stochastic approximation type algorithms with decaying stepsizes. 

Parallel to the development of stochastic approximation theory, some authors considered constant stepsizes where all $\beta_k = \beta$ are sufficiently small. In this case, $\{x_k\}$ forms a homogeneous Markov chain under suitable assumptions on $d_k$. This constant stepsize recursion is given by $x_{k+1} = f(x_k,w_k)$ for a suitably defined function $f$ that takes $x_k$ and i.i.d. noise $w_k$ as input. The convergence guarantee for this class of algorithms is usually weak --- the iterates may not converge to the desired solution $x^*$, but instead will do a random walk in some neighborhood of it. This class of recursive algorithm has been studied under the name of iterated random function systems \cite{dubins1966invariant,barnsley1989recurrent,diaconis1999iterated,duflo2013random,stenflo2012survey} and stochastic approximation with constant stepsize; see \cite{borkar2000ode,borkar2009stochastic,benaim1998recursive,roth2013stochastic} and the references therein for related discussions.

Stochastic optimization algorithms play a pivotal role in large-scale machine learning as well as data-driven/simulation-based learning problems \cite{bottou2018optimization}. For instance, stochastic gradient descent (SGD) is unarguably the most important class of algorithms
for many machine learning tasks. At the same time, simulation-based reinforcement learning has received significant attention \cite{silver2016mastering,silver2017mastering}. In both families of algorithms, randomization and sampling techniques are used in specialized ways to compute expectations that are otherwise expensive or intractable. 

Constant stepsize RSAs for optimization (in particular large-scale optimization) are widely studied. Constant stepsize SGD is studied with Markov chain methods in \cite{dieuleveut2017bridging} with respect to the Wasserstein distance. It is shown than an invariant distribution exists, and the distribution of the iterates converges geometrically to this invariant distribution. Further, it presents a formula for the concentration of the invariant distribution around the desired solution. Stochastic variance reduced gradient descent (SVRG) was designed to improve upon SGD (by using a variance reduced correction term), and it enjoys a linear convergence rate in expectation \cite{johnson2013accelerating}. SAGA is proposed in \cite{defazio2014saga} which offers an alternative variance reduction scheme. The hybrid algorithm HSAG combines the features of SVRG and SAGA and is developed in \cite{reddi2015variance}. SVRG and SAGA are extended to solve monotone inclusion problems in \cite{palaniappan2016stochastic} (which contains function minimization as a special case).

Within the reinforcement learning literature, constant stepsize stochastic approximation algorithms have been developed to compute the (approximately) optimal value function, the Q-value function, evaluate the performance (total discounted or average cost) of a stationary policy using temporal differences, and do all of these on top of function approximation. For finite-state finite-action discounted cost MDPs, constant step-size empirical value iteration (EVI) was studied in \cite{haskell2016}. The convergence guarantee of the algorithm was derived by constructing a Markov chain over a finite space and using a stochastic dominance argument to bound the error in the iterates. This methodology significantly departed from the ODE approach usually taken for proving convergence of constant stepsize stochastic approximation algorithms in \cite{borkar2000ode,borkar2009stochastic,benaim1998recursive}. This approach was later extended to average cost MDPs with empirical relative value iteration in \cite{gupta2018probabilistic}, where the dominating Markov chain was constructed over the space of natural numbers. Error bounds for constant step-size synchronous and asynchronous Q-learning algorithm were studied in \cite{beck2012error} by combining the union bound and triangle inequality. Finite-time bounds for temporal difference learning for evaluating stationary policies with constant stepsize have been obtained in \cite{srikant2019finite,bhandari2018finite} under a variety of assumptions.

\new{ Since the first writing of this paper, we learned of a new work \cite{amortila2020distributional,amortila2020pmlr} that viewed many reinforcement learning algorithms 
within the framework of iterated random operators. It showed that reinforcement learning algorithms such as temporal difference learning, optimistic policy iteration, Q-learning, etc. form a Markov chain. This work also established the geometric convergence of these Markov chains to their invariant distributions in the Wasserstein metric. We have thus omitted such results from this paper, even though they also fall within our framework.}

\subsection{Contributions}

We summarize our key contributions as follows: 
\begin{enumerate}
\item We model constant stepsize RSAs as iterated random function systems. Typically, a measure of distance, such as a metric or a divergence, between the iterates of the RSA $x_k$ and the desired solution $x^*$ is shown to have some kind of one-step contraction property (e.g. $\ex{\|x_k-x^*\|_2^2}$ is a common measure of distance for first-order optimization algorithms). Across RSAs, we noted that the measure of distance satisfies positive definiteness and symmetry -- two well-known properties enjoyed by a metric. However, they do not all satisfy the triangle inequality. Instead, these RSAs are Markov chains whose distributions satisfy some contraction property with respect to a certain {\it divergence}. We propose the notion of Wasserstein divergence to appropriately model this general phenomenon (see Definition \ref{def:Wasserstein divergence}).
\item We show that when the marginal distributions of the iterates of these RSAs contract with respect to the Wasserstein divergence: (i) there exists a unique invariant distribution for the RSA; and (ii) the marginal distribution converges to this invariant distribution geometrically with respect to the Wasserstein divergence (see Theorem \ref{thm:main}). These results allow us to assess the performance of the RSA in terms of how quickly it converges and how far the iterate is from $x^*$ after a sufficiently large runtime. 
\item We bound the concentration of the invariant distribution around the desired solution $x^*$ (see Theorem \ref{thm:concentration} and Theorem \ref{thm:contractionerror}). Furthermore, we show that variance reduced algorithms map the desired solution to itself almost surely. In this case, the invariant distribution for variance reduced RSAs is concentrated at $x^*$. This idea is used to establish the convergence of variance reduced algorithms such as SVRG, SAGA, and HSAG.
\item We develop several detailed examples of optimization algorithms that fall within our framework in Section \ref{sec:examples} (non-epoch based algorithms) and Section \ref{sec:examples-epoch} (epoch-based algorithms). We give specific attention to the illustrative quadratic case (see Table \ref{tab:quadratic divergence}), and then we extend to the nonlinear case (see Table \ref{tab:divergence nonlinear}).
\end{enumerate}

\new{
We note here two works \cite{hairer2011yet,ollivier2007ricci} that are closely related to the content of this paper\footnote{The authors would like to thank the associate editor for pointing us to these two references.}. In \cite{hairer2011yet}, the authors establish the existence of a unique invariant distribution of a Markov chain satisfying a geometric drift condition and a uniform minorization condition. While the result itself is well-known, the technique proposed by the authors to establish the result was novel. They construct a metric over the space of probability measures for which the Markov kernel (viewed as a linear operator over the space of probability measures) satisfies a certain contraction condition. This metric is built from the Lyapunov function of the Markov chain.

In \cite{ollivier2007ricci}, the author defines the coarse Ricci curvature associated with a random walk (which can be viewed as a Markov kernel) on a geodesic space. In particular, the coarse Ricci curvature is defined as 1 minus the ratio of the Wasserstein distance between the random walks originating from two different points on the manifold, and the distance between these two points. The author then proceeds to show that if the coarse Ricci curvature is bounded from below by a positive number for two nearby points on the manifold, then it is positive for any two points on the manifold. In this case, the random walk (a.k.a. Markov kernel) is a contraction operator over the space of probability measures with respect to the 1-Wasseretin metric.

Our present paper can be viewed as an extension of these ideas under the assumption that there is a divergence on the underlying state space of the Markov chain. It is this divergence that we lift to the space of measures to yield the corresponding Wasserstein divergence. Then, we frame all our convergence results for Markov kernels that are contractions with respect to this Wasserstein divergence.

\subsection{Outline of the Paper}
This paper is organized as follows. In Section \ref{sec:problem}, we frame our problem in terms of iterated random operators and introduce the notion of Wasserstein divergence, which is our key tool. We present our three main technical results in Section \ref{sec:mainresult}. In the following Sections \ref{sec:examples} and \ref{sec:examples-epoch}, we give several detailed examples of algorithms that fall within our framework (where we separate the non-epoch based algorithms from the epoch-based ones). We then proceed to establish the properties of the Wasserstein divergence in Section \ref{sec:wass}. The detailed proofs of our three main results are also presented in this section. Section \ref{sec:open} discusses implications of this work and open problems.

\subsection{Notation}
Let $(\ALP A,\rho_{\ALP A})$ be a complete separable metric (Polish) space with metric $\rho_{\ALP A}$. We use $\wp(\ALP A)$ to denote the set of all probability measures over the space $\ALP A$. A sequence of probability measures $\{\mu_k\}$ is said to converge in the weak* sense to a probability measure $\theta$ if and only if $\int fd\mu_k\to\int fd\theta$ as $k\to \infty$ for every continuous and bounded function $f:\ALP A\rightarrow\Re$ (this is sometimes referred to as weak convergence). We refer to \cite{villani2008optimal,ambrosio2008gradient,hernandez2012markov,ali2006,billingsley2013convergence} for more information on the weak* convergence of measures.

An operator $T:\ALP A\to\ALP A$ is a contraction if and only if there exists a contraction coefficient $\alpha\in[0,1)$ such that:
\beqq{\rho_{\ALP A}(T(a_1),T(a_2))\leq \alpha \rho_{\ALP A}(a_1,a_2),\, \forall a_1,a_2\in\ALP A.}
The Banach contraction mapping theorem asserts that any contraction operator $T$ over a complete metric space $\ALP A$ admits a unique fixed point $a^*$ such that $a^* = T(a^*)$. Moreover, starting with any $a_0\in\ALP A$, the iteration $\{a_k\}$ defined by $a_{k+1} = T(a_k)$ converges to $a^*$.

 We use $\wp(\ALP A)$ to denote the collection of all probability measures on the space $\ALP A$. We use $\FLD B(\ALP A)$ to denote the set of all Borel measurable subsets of $\ALP A$.
}

\section{Problem Formulation}\label{sec:problem}
\subsection{Preliminaries}

Let $\ALP X$ be a vector space with a metric $\rho$ so that $(\ALP X,\rho)$ is a complete and separable (Polish) space. Examples include: Euclidean spaces for any $p$-norm (where $p\geq 1$), Euclidean space with a weighted max norm, separable Banach and Hilbert spaces (e.g. $\ell_p$ spaces for $p\in[1,\infty)$), the space of continuous functions (in the supremum norm) over compact Hausdorff spaces, etc. Let $T:\ALP X\rightarrow\ALP X$ be a contraction operator with contraction coefficient $\alpha\in[0,1)$ and (unique) fixed point $x^*$.

We introduce another vector space $\ALP S$ in addition to $\ALP X$ (we will interpret $\ALP S$ as a ``lifting'' of $\ALP X$). For example, many variance reduced algorithms augment the original state space $\ALP X$ with additional information (e.g. the proxy terms in SAGA which store past gradient evaluations). The lifting $\ALP S$ allows us to cover these algorithms. For SVRG, we just have $\ALP S = \ALP X$. For SAGA, $\ALP S = \ALP X \times \ALP Y$ where $\ALP Y$ is the space of proxies for past gradient evaluations. We frame the rest of our discussion on $\ALP S$ to allow enough generality to cover all these cases. Abusing notation, we let $\rho$ also denote a metric on $\ALP S$ so that $(\ALP S,\rho)$ is a Polish space (the metric on $\ALP S$ is usually based on the metric on $\ALP X$ anyway).

When we operate on $\ALP S$, we are interested in computing $s^*$, which is an appropriate lifting of the desired fixed point $x^*$. For SVRG, $s^*$ is just $x^*$. For SAGA, $s^*$ is the concatenation of $x^*$ and all the gradient evaluations at $x^*$. In any case, we can recover $x^*$ from $s^*$.


The (classical) Wasserstein distance is defined next. Note that this definition is in terms of the metric $\rho$ on $\ALP S$. The theory of Wasserstein metric is covered in \cite{villani2008optimal,rachev1998i}. 

\begin{defn}
Let $p\in[1,\,\infty)$.

(i) $\mathcal{P}_{p}\left(\ALP S\right)$ is the set of all probability
measures on $\ALP S$ with finite $p^{th}-$order moment, i.e.,
those $\mu \in \wp(\ALP S)$ for which there exists $s_{0}\in\ALP S$
such that $\int_{\ALP S}\rho\left(s,\,s_{0}\right)^{p}\mu\left(ds\right)<\infty$.

(ii) For $\mu_1,\,\mu_2\in\mathcal{P}_{p}\left(\ALP S\right)$,
$C\left(\mu_1,\,\mu_2\right)$ is the collection of all $\xi \in \wp(\ALP S \times \ALP S)$ with marginals $\mu_1$
and $\mu_2$, i.e., $\xi\left(B\times\ALP S\right)=\mu_1\left(B\right)$
and $\xi\left(\ALP S\times B\right)=\mu_2\left(B\right)$ for
all $B\in\FLD B\left(\ALP S\right)$ (i.e., $C\left(\mu_1,\,\mu_2\right)$ is the set of all couplings with marginals $(\mu_1,\mu_2)$).

(iii) For $\mu_1,\,\mu_2\in\mathcal{P}_p\left(\ALP S\right)$, the $p-$Wasserstein distance
is
\[
W_{p}\left(\mu_{1},\,\mu_{2}\right)\triangleq\left(\inf_{\xi\in C\left(\mu_{1},\,\mu_{2}\right)}\int_{\ALP S\times\ALP S}\rho\left(s,\,s'\right)^{p}d\xi\left(s,\,s'\right)\right)^{1/p}.
\]
\end{defn}

\subsection{Wasserstein Divergence}
We now extend the previous definition (which depends on the metric $\rho$ on $\ALP S$) to accommodate divergence functions on $\ALP S$. We need this extension because the convergence analyses of many algorithms are done with respect to a divergence which is not a metric (and which does not satisfy the triangle inequality). We work with the following class of divergence functions in this paper.

\begin{defn}\label{def:divergence}
A function $V\text{ : }\mathcal{S} \times \mathcal{S}\rightarrow[0,\infty)$ is a divergence
function if $V$ is lower semi-continuous function and the following conditions hold:

(i) (Positive definiteness) $V\left(s_1, \, s_2\right)=0$ if and only if $s_1 = s_2$. 

(ii) (Symmetry) $V(s_1,s_2) = V(s_2,s_1)$ for all $s_1, \, s_2 \in \ALP S$.

(iii) (Inf-Compactness) For any $q \geq 0$ and compact set $\ALP K\subset \ALP S$, there exists a compact set $\ALP L\subset \ALP S$ such that
\beqq{\inf_{(s_1,s_2)\in \ALP L^\complement\times \ALP K} V(s_1,s_2)\geq q.}
The above inf-compactness condition is automatically satisfied if $\ALP S$ is a compact set and we have $\ALP L = \ALP S$ (since the infimum over an empty set is $\infty$).\hfill$\Box$
\end{defn}

\new{The following result is immediate.
\begin{lemma}
If $V$, $V_1$, and $V_2$ are divergences that satisfy all three conditions of Definition \ref{def:divergence}, then for any $p>0$, $V^p$, $V_1+V_2$, and $V^p+\rho$ are also divergences that satisfy all three conditions of Definition \ref{def:divergence}.
\end{lemma}
Some examples of divergences on $\ALP S = \Re^d$ that satisfy Definition \ref{def:divergence} include:\\
(a) $V(s_1,s_2) = (s_1 - s_2)^\top Q (s_1-s_2)$, where $Q \in \Re^{d \times d}$ is positive definite;\\
(b) $V(s_1,s_2) = \psi(s_1) + \psi(s_2) - 2\,\psi(s^*)$ for $s_1 \neq s_2$ and $V(s,s) = 0$ for all $s\in\ALP S$, where $\psi$ is a strongly convex function and $s^* = \arg\min_{s\in\ALP S} \psi(s)$;\\
(c) $V(s_1,s_2) = \|\nabla \psi(s_1) - \nabla \psi(s_2)\|_2^2$, where $\psi$ is a strongly convex function and $\|\cdot\|_2$ is the $\ell_2-$norm on $\Re^d$. 

\begin{rem}
We note here that $V(s_1,s_2)$ defined in (b) above is a valid metric on $\Re^d$ since it satisfies positive-definiteness, symmetry, and the triangle inequality\footnote{This has also been elaborated upon in \cite{hairer2011yet} in the context of Markov chain, where the Lyapunov function is used to construct such a metric. Our definition of divergence here is inspired by the metric constructed in \cite{hairer2011yet}.} However, $\Re^d$ equipped with the metric $V(\cdot,\cdot)$ is a complete metric space but it is not separable. Nonetheless, it satisfies all the requirements for being a divergence. 
\end{rem}

}

We use the divergence function introduced above to define the Wasserstein divergence.
\begin{defn}\label{def:Wasserstein divergence}
Let $V$ be a divergence function satisfying all three conditions of Definition \ref{def:divergence}.

(i) $\mathcal{P}_{V}\left(\ALP S\right)$ is the set of all
probability measures on $\ALP S$ with finite moment with respect to $V$, i.e., those $\mu \in \wp(\ALP S)$ for which there exists
$s_{0}\in\ALP S$ such that
$\int_{\ALP S}V\left(s,\,s_{0}\right)\mu\left(ds\right)<\infty$.

(ii) For $\mu_1,\,\mu_2\in\mathcal{P}_{V}\left(\ALP S\right)$, the Wasserstein divergence is
\[
W_V\left(\mu_1,\,\mu_2\right)\triangleq \inf_{\xi \in C\left(\mu_1,\,\mu_2\right)}\int_{\ALP S\times\ALP S}V\left(s,\,s'\right)d\xi\left(s,\,s'\right).
\]
\end{defn}

\subsection{Iteration of random operators}

We express constant stepsize RSAs as iteration of random operators on $\ALP S$, which we now formalize. Let $(\Omega,\ALP F,\mathbb{P})$ be a probability space with filtration $\{\ALP F_k\}_{k \geq 0}$. We consider a collection of random operators $\{\hat T_k\}_{k \geq 0}$ defined on $(\Omega,\ALP F,\mathbb{P})$, where each $\hat T_k:\Omega\times\ALP S\rightarrow \ALP S$ is an $\ALP F_k-$adapted operator-valued random variable. Throughout the paper, we will implicitly make the following assumption.
\new{
\begin{assumption}
\label{i.i.d.}
The sigma algebras $\sigma\{\hat T_k(\cdot,s), k\in\Na, s\in\ALP S\}$ are independent.
\end{assumption}
}
Under Assumption \ref{i.i.d.}, we have a Markov chain $\left\{s_{k}\right\}_{k\geq0}$ produced by:
\begin{equation}
s_{k+1}=\hat{T}_{k}\left(s_{k}\right)\triangleq\hat{T}_{k}\left(\omega,s_{k}\right),\,\forall k\geq0,\label{eq:iteration}
\end{equation}
where we usually leave the dependence on $\omega$ implicit. By assumption on $\{\hat T_k\}_{k \geq 0}$, each $s_k$ is $\ALP F_k-$adapted. \new{Since $\ALP S$ is Polish space, due to \cite[Theorem 2.8]{bhattacharya2007basic} or \cite[Theorem 4.34]{breiman1992probability}, there is a transition kernel $\mathfrak{Q}$ on $\mathcal{S}$ corresponding to Eq. (\ref{eq:iteration}) which satisfies}:
\[
\mathfrak{Q}\left(s_{k},\,B\right) \triangleq \text{Pr}\left\{ s_{k+1}\in B\,\vert\,s_{k}\right\} ,\,\forall B\in\mathcal{B}\left(\mathcal{S}\right).
\]
We recursively define the $k-$step transition kernels:
\[
\mathfrak{Q}^{k+1}\left(s_0,\,B\right) \triangleq \int_{\ALP S} \mathfrak{Q}^k\left(s_0,\,ds\right)\mathfrak{Q}\left(s,\,B\right),\,\forall B\in\mathcal{B}\left(\mathcal{S}\right),
\]
for all $k\geq0$. We let $\mu_k \in \wp(\ALP S)$ denote the marginal distribution of $s_k$ which satisfies $\mu_k = \mu_0 \mathfrak Q^k$, for all $k\geq0$.

When we characterize the evolution of $\{s_{k}\}_{k\geq0}$, we can talk about either dynamic $s_{k+1}=\hat{T}_{k}\left(s_{k}\right)$ or $\mu_{k+1} = \mu_k\mathfrak Q$ -- both concepts are equivalent. When we characterize the evolution of the \textit{distribution} of $s_{k}$, then the dynamic $\mu_{k+1} = \mu_k \mathfrak Q$ is more useful. 

\new{In the following result, we show that if $\{\hat T_k\}$ enjoys a contraction property on average with respect to the divergence $V$, then the corresponding transition kernel is a contraction with respect to the Wasserstein divergence.

\begin{theorem}
\label{thm:conditional_expectation}
[Proof in Subsection \ref{sub:conditional}]
The following holds:

(i)  Let $\mu_{1},\,\mu_{2}\in\mathcal{P}_V\left(\mathcal{S}\right)$. If (a) $V$ is a lower semicontinuous and positive-definite divergence (that is, satisfying condition (i) in Definition \ref{def:divergence}); and (b) there exists $\alpha\in\left(0,\,1\right)$ such that:
\[
\mathbb{E}\left[V\left(\hat{T}_{k}\left(s\right),\,\hat{T}_{k}\left(s'\right)\right)\right]\leq\alpha\,V\left(s,\,s'\right),\,\forall s,\,s'\in\mathcal{S},
\]
for all $k\geq0$, then $W_V\left(\mu_{1}\mathfrak{Q},\,\mu_{2}\mathfrak{Q}\right)\leq\alpha\,W_V\left(\mu_{1},\,\mu_{2}\right)$.

(ii) Let $\mu\in\ALP P_V(\ALP S)$, $f:\ALP S\to\ALP S$ be a measurable map, and $\mu\circ f^{-1}$ be a pullback of $\mu$. Then, $W_V(\mu,\mu\circ f^{-1})\leq \ex{V(s,f(s))}$, where $s$ is distributed according to $\mu$.

(iii) If $\mu\in\ALP P_V(\ALP S)$ and $s^*\in\ALP S$, then $W_V(\mu,\ind{s^*}) = \ex{V(s,s^*)}$, where $s$ is distributed according to $\mu$.
\end{theorem}
\noindent
Theorem \ref{thm:conditional_expectation} also holds in the special case where $V = \rho$ is a metric.
}

\section{Main Results}\label{sec:mainresult}

We will encode the constant step size RSAs under study in the form of Eq.
(\ref{eq:iteration}). \new{Given an RSA, we first need to identify the state space of the Markov chain. We construct in Sections \ref{sec:examples} and \ref{sec:examples-epoch} the state spaces for various RSAs used in optimization.} Then, we will analyze the behavior of $\left\{ s_{k}\right\} _{k\geq0}$ in two steps:
\begin{enumerate}
\item (Contraction) Show that there exist a divergence function $V$ and a constant $\alpha\in[0,\,1)$ such that the relationship
\begin{equation}
W_V\left(\mu_{1}\mathfrak{Q},\,\mu_{2}\mathfrak{Q}\right)\leq\alpha\,W_V\left(\mu_{1},\,\mu_{2}\right)\,\text{ for all } \mu_{1},\,\mu_{2}\in\mathcal{P}_V\left(\mathcal{S}\right),\label{eq:contraction}
\end{equation}

holds. Equation (\ref{eq:contraction}) is effectively a contraction in the Wasserstein divergence. This property implies the existence of an invariant distribution $\vartheta$
for the Markov chain $\left\{ s_{k}\right\} _{k\geq0}$. It also implies a geometric convergence
rate of the marginal distributions $\mu_k = \mu\,\mathfrak{Q}^{k}$ to $\vartheta$ with respect to the Wasserstein divergence, for any initial distribution
$\mu\in\mathcal{P}_V\left(\mathcal{S}\right)$ (see Theorem \ref{thm:main}).
\item (Concentration) Demonstrate the concentration of $\vartheta$ around the desired $s^{*}$ (see Theorem \ref{thm:concentration} and Theorem \ref{thm:contractionerror}).
\end{enumerate}
Many RSAs satisfy the above contraction property, but are not concentrated at $s^{*}$ (e.g. SGD). So, both of these steps are essential.

\subsection{Contraction in Wasserstein divergence}

We first discuss the implications of contraction in the Wasserstein
divergence. Our main assumption is formalized next, it characterizes the key
contraction property of $\{ \hat{T}_{k}\} _{k\geq0}$.
\begin{assumption}\label{assu:contraction}
(i) There exists a divergence function $V$ satisfying Definition \ref{def:divergence} and a constant $\alpha\in\left(0,\,1\right)$ such that:
\[
W_V\left(\mu_{1}\mathfrak{Q},\,\mu_{2}\mathfrak{Q}\right)\leq\alpha\,W_V\left(\mu_{1},\,\mu_{2}\right),\,\forall\mu_{1},\,\mu_{2}\in\mathcal{P}_V\left(\mathcal{S}\right).
\]
(ii) There exists a nonempty set $\ALP M\subset\ALP P_V(\ALP S)$ such that for any $\mu\in\ALP M$, there exists $c_\mu<\infty$ such that $\sup_{k\geq 0}W_V(\mu\,\mathfrak{Q}^k,\mu)\leq c_\mu$.
\end{assumption}
Assumption \ref{assu:contraction}(i) states that the marginal distributions of the sequences starting from two different initial distributions mix at a geometric rate. We show in Section \ref{sec:wass} that this assumption is trivially satisfied if the operator $\hat T_k$ is itself a contraction with respect to the divergence function $V$. Assumption \ref{assu:contraction}(ii) is automatically satisfied if $\ALP S$ is a compact set, or if the Markov chain $\{s_k\}_{k\geq 0}$ is uniformly bounded almost surely. A simple approach to force boundedness of the Markov chain is to project the Markov chain back onto some large ball if the chain wanders away from this ball. In practice, such projections are not required and some amount of hyperparameter tuning is done to avoid this situation.

Under Assumption \ref{assu:contraction}, there is an invariant distribution $\vartheta$ for $\{s_k\}_{k \in \Na}$, and the marginal distributions of $s_k$ converge geometrically to $\vartheta$ with respect to the Wasserstein divergence.
\begin{theorem}
\label{thm:main} 
[Proof in Subsection \ref{sub:main}] Suppose Assumption \ref{assu:contraction} holds and let $\{s_k\}_{k \in \Na}$ be produced by Eq. (\ref{eq:iteration}).

(i) For all $k\geq0$, we have
\[
W_V\left(\mu_{1}\mathfrak{Q}^{k},\,\mu_{2}\mathfrak{Q}^{k}\right)\leq\alpha^{k}W_V\left(\mu_{1},\,\mu_{2}\right),\,\forall\mu_{1},\,\mu_{2}\in\mathcal{P}_V\left(\mathcal{S}\right).
\]

(ii) There exists a unique invariant distribution $\vartheta$ for $\left\{ s_{k}\right\} _{k\geq0}$ satisying $\vartheta\, \mathfrak{Q} = \vartheta$.

(iii) For all $k\geq0$ and $\mu\in\ALP M$, we have $W_V\left(\mu\,\mathfrak{Q}^{k},\,\vartheta\right)\leq\alpha^{k}W_V\left(\mu,\,\vartheta\right)$.

\new{(iv) For any $\mu\in\ALP M$, $\mu\,\mathfrak Q^k$ converges to $\vartheta$ in the weak* sense.}
\end{theorem}

Theorem \ref{thm:main} is analogous to the Banach fixed point theorem, albeit with respect to a divergence rather than a metric. When $\{\hat T_k\}_{k \in \Na}$ satisfy Assumption \ref{assu:contraction}, then Theorem \ref{thm:main} establishes the existence of a ``fixed point'' (in the sense of the
invariant distribution $\vartheta$ of the RSA) and it also establishes a linear convergence rate of the sequence of distributions of the sequence $(s_k)_{k \in \Na}$ to this invariant distribution (with respect to the Wasserstein divergence). While Assumption \ref{assu:contraction}(i) is the usual contraction condition and is the only assumption required for the Banach contraction mapping theorem to hold (in a Polish space), Assumption \ref{assu:contraction}(ii) allows us to establish this result when the operator over the underlying space $\mathcal{S}$ is contraction with respect to a divergence.

\subsection{Concentration}
Next we see that the concentration of $\vartheta$ around $s^*$ directly depends on the action of $\{\hat T_k\}_{k \in \Na}$ on $s^*$.

\begin{assumption}
\label{assu:concentration}
For all $k \geq 0$, $\hat T_k(s^*)=s^*$ almost surely.
\end{assumption}

This assumption is satisfied for many variance reduced algorithms (e.g. SVRG, SAGA, HSAG, etc.). Assumption \ref{assu:concentration} is equivalent to saying that $\ind{s^*} \mathfrak Q = \ind{s^*}$, or that the transition kernel $\mathfrak Q$ always maps $s^*$ back to itself with probability one. It must then be that $\vartheta = \ind{s^*}$, since the invariant distribution of $\mathfrak Q$ is unique, and we get the following conclusion. We note that the following theorem does not require $V$ to be a metric.

\begin{theorem}\label{thm:concentration}
Suppose Assumptions \ref{assu:contraction}(i) and \ref{assu:concentration} hold, then $\lim_{k\rightarrow\infty} W_V\left(\mu\,\mathfrak{Q}^{k},\,\ind{s^*}\right) = 0$ (i.e., $\vartheta = \ind{s^*}$). Moreover, the rate of convergence is geometric in the Wasserstein divergence.
\end{theorem}
\begin{proof}
From Assumption \ref{assu:concentration}, we know that $\ind{s^*}\mathfrak{Q} = \ind{s^*}$. Due to Assumptions \ref{assu:contraction}(i), we conclude that for any $\mu\in\ALP P_V(\ALP S)$ such that $W_V(\mu,\ind{s^*})<\infty$, we have $W_V(\mu\mathfrak Q^k,\ind{s^*}) \leq \alpha^k W_V(\mu,\ind{s^*})\leq W_V(\mu,\ind{s^*})<\infty$. Thus, Assumption \ref{assu:contraction}(ii) also holds, which implies Theorem \ref{thm:main} is applicable and we have $\vartheta\, \mathfrak{Q} = \vartheta$. This yields
\beqq{W_V\left(\ind{s^*},\vartheta\right) = W_V\left(\ind{s^*}\mathfrak{Q},\vartheta\mathfrak{Q}\right) \leq \alpha W_V\left(\ind{s^*},\vartheta\right).}
Since $\alpha<1$, the above expression immediately implies that $W_V\left(\ind{s^*},\vartheta\right) = 0$, and so $\ind{s^*}$ is the invariant distribution (this follows from Proposition \ref{prop:WVpproperties}(ii) which is proved later).
\end{proof}

\new{
\subsection{Contraction with Errors}
In some RSAs, the random operator may be contracting on average with respect to a divergence but with an additive error. For example, function approximators in reinforcement learning introduce such errors -- this has been elucidated in \cite{munos2008finite} in the context of fitted value iteration.  The next two lemmas give error bounds in the Wasserstein divergence. The proofs of both results follow from the principle of mathematical induction.

\begin{lemma}\label{lem:error1}
Suppose there exists a divergence function $V$ satisfying Definition \ref{def:divergence}, $\alpha\in(0,1)$, and $\epsilon>0$ such that $W_V(\mu_1\mathfrak Q,\mu_2\mathfrak Q) \leq \alpha W_V(\mu_1,\mu_2)+\epsilon$. Then, 
\beqq{W_V(\mu_1\mathfrak Q^k,\mu_2\mathfrak Q^k) \leq \alpha^k W_V(\mu_1,\mu_2)+\left(\frac{1-\alpha^k}{1-\alpha}\right)\epsilon.}
\end{lemma}

\begin{lemma}\label{lem:error2}
Suppose there exists a divergence function $V$ satisfying Definition \ref{def:divergence}, a constant $\alpha\in(0,1)$, and $\epsilon>0$ such that
\beq{\label{eqn:wvs} W_V(\mu\,\mathfrak Q,\ind{s^*}) \leq \alpha W_V(\mu,\ind{s^*})+\epsilon.}
Then, 
\beqq{W_V(\mu\,\mathfrak Q^k,\ind{s^*}) \leq \alpha^k W_V(\mu_1,\ind{s^*})+\left(\frac{1-\alpha^k}{1-\alpha}\right)\epsilon.}
\end{lemma}

The following theorem considers concentration of an RSA which is a contraction in the Wasserstein divergence with error.
\begin{theorem}\label{thm:contractionerror}
Suppose: (i) Assumption \ref{assu:contraction} holds for a divergence function $\tilde V$ satisfying Definition \ref{def:divergence}; and (ii) Eq. \eqref{eqn:wvs} holds for a lower semicontinuous divergence function $V$. Then, a unique invariant measure $\vartheta$ exists and $W_V(\vartheta,\ind{s^*})\leq \epsilon/(1-\alpha)$. Consequently, for any $\kappa>0$ we have
\beq{\label{eqn:markov}\lim_{k\rightarrow\infty} \pr{V(s_k,s^*)\geq \kappa}\leq \frac{\epsilon}{\kappa(1-\alpha)}.}
\end{theorem}
\begin{proof}
The proof is an immediate consequence of Theorem \ref{thm:main} and Lemma \ref{lem:error2}. The assertion in Eq. \eqref{eqn:markov} is a direct application of Markov's inequality, where the limit exists because there is an invariant distribution $\vartheta$.  
\end{proof}

\begin{rem}
The divergence functions for the two hypothesis of Theorem \ref{thm:contractionerror} are allowed to be different.
Furthermore, the function $\tilde V$ could be a metric on the space $\ALP S$ so that $\ALP S$ is a complete metric space. In this case, the convergence to the invariant distribution is in the 1-Wasserstein metric.
\end{rem}
}

\new{
\section{Examples in Optimization}\label{sec:examples}
We consider several RSAs for function minimization in this section, starting with the illustrative quadratic case and then extending to the general nonlinear case. For each algorithm, we first identify the state space $\ALP S$ and the corresponding random operators $\{\hat T_k\}_{k \in \Na}$ on $\ALP S$. 
We take two different starting points $s_0^{(i)}$ for $i = 1, 2$ and iteratively apply the same sequence of random operators to obtain the sequences $s_{k+1}^{(i)} = \hat{T}_k\big(s_{k}^{(i)}\big)$ for all $k\geq 0$ for $i = 1, 2$. We then identify an appropriate divergence function $V$, and show that $\hat T_k$ satisfies the contraction condition of Theorem \ref{thm:conditional_expectation}. As a consequence of Theorem \ref{thm:conditional_expectation}, we conclude that Assumption \ref{assu:contraction}(i) is satisfied. We can further use the following lemma to show that that Assumption \ref{assu:contraction}(ii) holds so we may apply our main Theorem \ref{thm:main}. 
\begin{lemma}\label{lem:K}
Suppose there exists a set $\ALP K\subset\ALP S$ such that: (a) the diameter of $\ALP K$, defined as $\mathbb{D}_V(\ALP K)\triangleq\sup_{s,s'\in\ALP K} V(s,s')<\infty$; and (b) $\hat T_k(s)\in\ALP K$ almost surely for all $s\in\ALP K$ and $k\in\Na$. Then, $W_V(\mu,\mu\mathfrak Q^k)\leq \mathbb{D}_V(\ALP K)$ for all $k\in\Na$ for any $\mu$ with support in $\ALP K$. 
\end{lemma}
\begin{proof}
We know that $\ALP K$ is bounded by hypothesis (a). Due to hypothesis (b), we conclude that for any $\mu\in\ALP P_V(\ALP S)$ such that the support of $\mu$ is in $\ALP K$, the support of $\mu\mathfrak Q^k$ is also in $\ALP K$. This immediately yields $W_V(\mu,\mu\mathfrak Q^k)\leq \mathbb{D}_V(\ALP K)$ for all $k\in\Na$.
\end{proof}

In practice, $\ALP K$ can be a large ball. For each of the algorithms studied in the sequel, the two hypotheses of Lemma \ref{lem:K} are satisfied, and so Assumption \ref{assu:contraction}(ii) is also satisfied. For later use, we define the difference sequence $\{\Delta s_k\}_{k \geq 0}$ where $\Delta s_k \triangleq s_k^{(1)} - s_k^{(2)}$ for all $k \geq 0$.

\subsection{The quadratic case}
\label{sub:quadratic}

The quadratic case enjoys a special property, where the difference sequence $\{\Delta s_k\}_{k \geq 0}$ produced by an RSA coincides with the same algorithm applied to a quadratic problem with optimal solution zero. The rate of mixing then follows immediately from the rate at which $\{\Delta s_k\}_{k \geq 0}$ converges to zero (in some divergence).

\begin{table}[]
    \centering
        \begin{tabular}{|M{0.13\linewidth}|P{0.63\linewidth}|M{0.13\linewidth}|}
        \hline
        Algorithm & \centering State Space/Divergence & Reference\\
        \hline
        \multirow{1}{*}{SGD} & $s_k = x_k $\newline $V(s_k,s_k')= \|s_k - s_k'\|_2^2$ & \cite{robbins1951stochastic}
        \\\cline{1-3}
        \multirow{1}{*}{ASGD} & $s_k = (x_k,\,x_{k-1}) $\newline $V(s,s')= \|s_k - s_k'\|_{P_{\alpha,\,\beta}}^2 + \|x_k - x_k'\|_{Q}^2$ & \cite{can2019accelerated}
        \\\cline{1-3}
        \multirow{1}{*}{SAGA} & $s_k = (x_k,\,\varphi_k)$\newline $V_{b}(s_k,s_k')= \|x_k-x_k'\|_2^2+b\sum_{n \in [N]}\|Q_n\varphi_{k,\,n} - Q_n\varphi_{k,\,n}'\|_2^2$ & \cite{defazio2014saga}
        \\\cline{1-3}
        \multirow{1}{*}{ASVRG} & $s_k = x_k $\newline $V(s_k,s_k')= \|s_k - s_k'\|_Q^2$ & \cite{shang2018asvrg}
        \\\cline{1-3}
        \end{tabular}
        \vspace{0.5em}
    \caption{\label{tab:quadratic divergence} Divergences and corresponding recursive stochastic optimization algorithm for the quadratic case. Here, we define $\|y\|_Q^2 \triangleq y^\top Q\, y$ for a positive definite matrix $Q$. }
\end{table}

\subsubsection{Oracle-based stochastic gradient descent (SGD)}

We start with the quadratic minimization problem
\begin{equation}\label{quadratic oracle}
    \min_{x \in \mathbb R^d} f(x) \triangleq \left\{ \frac{1}{2}x^\top Q\, x +a^\top x + b \right\},
\end{equation}
where $c\, I_d \preceq Q \preceq L\, I_d$ and $0 < c \leq L$ (so $Q$ is positive definite). For later reference, we also define the quadratic minimization problem
\begin{equation}\label{quadratic oracle-1}
    \min_{x \in \mathbb R^d} \frac{1}{2}x^\top Q\, x,
\end{equation}
which has optimal solution $x^* = 0$.

Suppose that there is a sequence of i.i.d. zero-mean uniformly bounded noise $\{\varepsilon_k\}_{k \geq 0}$ where each call to $\nabla f$ returns $\nabla f(x_k) + \varepsilon_k$. Then, we have $\ALP S = \Re^d$, $s_k = x_k$, and
$$
    \hat{T}_k(x_k) \triangleq x_k - \eta [(Q\,x_k + a) + \varepsilon_k],\,\forall k \geq 0,
$$
where $\eta > 0$ is the stepsize. By linearity of $\hat{T}_k$, we have:
\begin{equation}\label{eq:sgd coupled}
    \Delta x_{k+1} = x_{k+1}^{(1)} - x_{k+1}^{(2)} = \hat{T}_k(x_k^{(1)}) - \hat{T}_k(x_k^{(2)}) = T(\Delta x_k),\,\forall k \geq 0,
\end{equation}
where $T(x) \triangleq x - \eta\, Q\,x$. We observe that the sequence $\{\Delta x_k\}_{k \geq 0}$ produced by Eq. \eqref{eq:sgd coupled} corresponds to the sequence produced by applying (exact) gradient descent to Problem \eqref{quadratic oracle-1}. In particular, we have
\begin{equation}
    \|x_{k+1}^{(1)} - x_{k+1}^{(2)}\|_2^2 = \|\Delta x_{k+1}\|_2^2 \leq \gamma(\eta) \|\Delta x_{k}\|_2^2 = \gamma(\eta) \|x_k^{(1)} - x_k^{(2)}\|_2^2,\,\forall k \geq 0,
\end{equation}
almost surely where $\gamma(\eta) \triangleq 1 - 2 \eta\, c + \eta^2 L^2$. It follows immediately that Assumption \ref{assu:contraction}(i) holds for this case due to Theorem \ref{thm:conditional_expectation}(i). We may then apply Theorem \ref{thm:main} to conclude the existence of an invariant distribution $\vartheta$ such that $W_V(\mu_k,\,\vartheta) \leq \gamma(\eta)^k W_V(\mu_0,\,\vartheta)$.

\subsubsection{Oracle-based accelerated SGD (ASGD)}

We now consider ASGD for Problem \eqref{quadratic oracle}. We let $\ALP S = \Re^d\times\Re^d$, $s_k = (x_k,\,x_{k-1})$, and
$$
\hat T_k(s_k^{(i)}) = ((1 + \beta)x_k^{(i)} - \beta\, x_{k-1}^{(i)} - \eta\, [\nabla f(H\, s_k^{(i)}) + \varepsilon_{k}^{(i)}],\,x_k^{(i)})
$$
where $H\, s_k^{(i)} = (1+\alpha)x_k^{(i)} - \alpha\, x_{k-1}^{(i)}$ for $i = 1, 2$. This iteration recovers classical ASGD for $\alpha = \beta$ and the stochastic heavy-ball method for $\alpha = 0$ (see \cite{hu2018dissipativity,can2019accelerated}).
Let $P_{\alpha,\,\beta} \in \mathbb{S}_+^{d}$ and $\rho_{\alpha,\,\beta} \in (0,\,1)$ satisfy the LMI
$$
\left(\begin{array}{cc}
A^\top P\, A - \rho\,P & A^\top P\, B\\
B^\top P\, A & B^\top P\, B
\end{array}\right) - X \preceq 0,
$$
where $A = \tilde A \otimes I_d$ and $B = \tilde B \otimes I_d$ for
\[
\tilde A = \left[\begin{array}{cc}
1 + \beta & -\beta\\
1 & 0
\end{array}\right] \text{ and }
\tilde B = \left[\begin{array}{c}
-\alpha\\
0
\end{array}\right],
\]
as in \cite[Theorem 1]{can2019accelerated}, and $X = X_1 + \rho^2 X_2 + (1 - \rho^2) X_3 \in \mathbb{R}^{2d \times 2d}$ is defined as in \cite[Lemma 5]{hu2018dissipativity}. We define the corresponding divergence
\begin{equation}
    V(s_k^{(1)},\,s_k^{(2)}) = (s_k^{(1)} - s_k^{(2)})^\top P_{\alpha, \beta} (s_k^{(1)} - s_k^{(2)}) + \frac{1}{2} (x_k^{(1)} - x_k^{(2)})^\top Q (x_k^{(1)} - x_k^{(2)}).
\end{equation}
It follows that
$\mathbb{E}[V(s_{k+1}^{(1)},\,s_{k+1}^{(2)})\,\vert\,\mathcal F_k] \leq \rho_{\alpha,\,\beta} V(s_{k}^{(1)},\,s_{k}^{(2)})$,
where $\rho_{\alpha,\,\beta} \in (0,1)$, which aligns with \cite[Lemma 20]{can2019accelerated} for classical ASGD. Theorem \ref{thm:main} gives the existence of an invariant distribution $\vartheta$ such that $W_V(\mu_k,\,\vartheta) \leq \rho_{\alpha,\,\beta}^k W_V(\mu_0,\,\vartheta)$.

\subsubsection{SAGA}

SAGA (\cite{defazio2014saga}) is based on ``proxies'' that store past gradient evaluations. Let $[N]:= \{1,\,2,\ldots,\,N\}$ be a finite index set and consider the finite sum problem
\begin{equation}\label{quadratic finite sum}
    \min_{x \in \mathbb R^d} f(x) \triangleq \left\{ \frac{1}{N}\sum_{n \in [N]} f_n(x) \triangleq \left\{ \frac{1}{2}x^\top Q_n\, x +a_n^\top x + b_n \right\} \right\},
\end{equation}
where $c\, I_d \preceq Q_n \preceq L\, I_d$ for all $n \in [N]$ and $0 < c \leq L$ (so all $\{Q_n\}$ and their average $\frac{1}{N}\sum_{n \in [N]} Q_n$ are positive definite). We also introduce the associated finite sum problem
\begin{equation}\label{quadratic finite sum-1}
    \min_{x \in \mathbb R^d} \frac{1}{N}\sum_{n \in [N]}^N  \frac{1}{2}x^\top Q_n\, x,
\end{equation}
which has optimal solution $x^* = 0$.

Now let $\ALP S = \Re^d \times \Re^{N\,d}$ and $s_k = (x_k, \varphi_k)$ where $\varphi_k = (\varphi_{k,\,n})_{n \in [N]}$ and $\varphi_{k,\,n}$ corresponds to the point where $\nabla f_n$ was last evaluated at or before iteration $k \geq 0$. Let $\{I_k\}_{k \geq 0}$ be a sequence of i.i.d. uniform RV's on $[N]$ and define $\hat{T}_k(s) = (\hat G_k(s),\,\hat U_k(s))$ where
\beq{\hat{G}_k(s) &= x_k - \eta \left( Q_{I_k}x - Q_{I_k}\varphi_{k,\,I_k} + \frac{1}{N} \sum_{n=1}^N (Q_n \varphi_{k,\,n} + a_n)\right),\\
\hat{U}_k(s) &= \begin{cases}
x_k & I_k = n,\\
\varphi_{k,\,n} & \text{otherwise}.
\end{cases}\label{eqn:Uk}
}

Then, we see that the difference sequence $\Delta s_{k+1} = \hat{T}_k(\Delta s_k)$ exactly corresponds to SAGA applied to Problem \eqref{quadratic finite sum-1}. For $b > 0$, we introduce the function
$$
V_{b}(s_k^{(1)},s_k^{(2)})= \|x_k^{(1)}-x_k^{(2)}\|_2^2+b\sum_{n \in [N]}\|Q_n\varphi_{k,\,n}^{(1)} - Q_n\varphi_{k,\,n}^{(2)}\|_2^2,
$$
which we verify is a divergence for the general nonlinear case in Lemma \ref{lem:divergence SAGA}. For
\[
\alpha\left(\eta\right)=\max\left\{ \gamma\left(\eta\right)+b\,L^{2},\,\frac{\eta^{2}/b+N-1}{N}\right\},
\]
we have $\mathbb{E}[V_{b}(s_{k+1}^{(1)},s_{k+1}^{(2)}) \, \vert \, \mathcal{F}_k] \leq \alpha(\eta) V_{b}(s_{k}^{(1)},s_{k}^{(2)})$ for all $k \geq 0$. By Theorem \ref{thm:main}, there is an invariant distribution $\vartheta$ such that $W_V(\mu_k,\,\vartheta) \leq \alpha(\eta)^k W_V(\mu_0,\,\vartheta)$. Assumption \ref{assu:concentration} holds and so $\vartheta = \ind{s^*}$ by Theorem \ref{thm:concentration}, where $s^* = (x^*,\,(\varphi_n^*)_{n \in [N]})$ and $\varphi_n^* = x^*$ for all $n \in [N]$.

\subsection{The nonlinear case}
\label{sub:nonlinear}

Suppose $\{f_n\}_{n \in [N]}$ are all $c-$strongly convex, differentiable, and have $L-$Lipschitz gradients. Let $g$ be convex and suppose its proximal mapping:
$$
\text{prox}_{g}(z):= \arg \min_{x \in \mathbb{R}^d} \left\{ g(x) + \frac{1}{2}\|x - z\|_2^2 \right\},
$$
is tractable. The finite sum minimization problem is then:
\begin{equation}\label{eq:finite_sum}
\min_{x \in \mathbb R^d} \psi(x):= \left\{f(x) + g(x) \right\} \text{ where } f(x) := \frac{1}{N}\sum_{n \in [N]} f_n(x),
\end{equation}
where $\psi$ is $c-$strongly convex by assumption on $\{f_n\}_{n \in [N]}$.  By strong convexity of $\psi$, there is a unique optimal solution $x^* \in \mathbb R^d$ of Problem \eqref{eq:finite_sum} with optimal value $\psi^* = \psi(x^*)$.

\begin{table}[]
    \centering
        \begin{tabular}{|M{0.13\linewidth}|P{0.63\linewidth}|M{0.13\linewidth}|}
        \hline
        Algorithm & \centering State Space/Divergence & Reference\\
        \hline
        \multirow{1}{*}{SAGA} & $s_k = (x_k,\,\varphi_k) $\newline $V_{b}(s_k,s_k')= \|x_k-x_k'\|_2^2+b\sum_{n=1}^N\|\nabla f_n(\varphi_{k,n}) - \nabla f_n(\varphi_{k,n}')\|_2^2$ & \cite{defazio2014saga}
        \\\cline{1-3}
        \multirow{1}{*}{HSAG} & $s_k = (x_k,\,(\varphi_{k,n})_{ n \in S})$\newline $V_{b,S}(s_k,s_k')= \|x_k-x_k'\|_2^2+b\sum_{n \in S}\|\nabla f_n(\varphi_{k,n}) - \nabla f_n(\varphi_{k,n}')\|_2^2$ & \cite{reddi2015variance}
        \\\cline{1-3}
        ASVRG & $s_k = x_k $\newline $V(s_k,s_k')= \begin{cases}
        0 & s_k = s_k'\\\psi(s_k) + \psi(s_k') - 2\,\psi(x^*) & s_k \ne s_k' \end{cases}$ & \cite{shang2018asvrg}
        \\\cline{1-3}
        Catalyst & $s_k = (x_k,\,x_{k-1}) $\newline $\bar{V}(s_k,\,s_k') = \begin{cases}
        0 & s_k = s_k'\\
        V(x_k,\,x_k') + (1 - \alpha)V(x_{k-1},\,x_{k-1}') & s_k \ne s_k'
        \end{cases}$ & \cite{lin2017catalyst}
        \\\cline{1-3}
        \end{tabular}
        \vspace{0.5em}
    \caption{\label{tab:divergence nonlinear} Divergences and corresponding recursive stochastic optimization algorithm for the nonlinear case. }
\end{table}

\subsubsection{Stochastic gradient descent (SGD)}

For SGD, we have $\ALP S = \Re^d$, $s_k = x_k$, and
\begin{equation}
    \hat{T}_k(x) = \text{prox}_{\eta\, g}(x- \eta \nabla f_{I_k}(x)),\,\forall k \geq 0.
\end{equation}
By non-expansiveness of the proximal mapping, we have
\begin{equation}
    \|x_{k+1}^{(1)} - x_{k+1}^{(2)}\|_2^2 \leq \gamma(\eta) \|x_{k}^{(1)} - x_{k}^{(2)}\|_2^2,\,\forall k \geq 0.
\end{equation}
This same reasoning applies to batch gradient descent. For each $k \geq 0$, let $\{I_{k,\,j}\}_{j = 1}^J$ be an i.i.d. sample of size $J \geq 1$ from $[N]$ (with replacement). We may then define
\begin{equation}
    \hat{T}_k(x) = \text{prox}_{\eta\, g}\left(x- \frac{\eta}{J}\sum_{j=1}^J\nabla f_{I_{k,\,j}}(x)\right),\,\forall k \geq 0,
\end{equation}
where again we have $\|x_{k+1}^{(1)} - x_{k+1}^{(2)}\|_2^2 \leq \gamma(\eta) \|x_{k}^{(1)} - x_{k}^{(2)}\|_2^2$ for all $k \geq 0$. For any $J \geq 1$, Theorem \ref{thm:main} gives an invariant distribution $\vartheta$ such that $W_V(\mu_k,\,\vartheta) \leq \gamma(\eta)^k W_V(\mu_0,\,\vartheta)$. The sample size $J \geq 1$ effects the concentration of the $\vartheta$ around the optimal solution $s^* = x^*$.

\subsubsection{SAGA}

Let $\ALP S = \Re^d \times \Re^{N\,d}$, $s_k = (x_k,\,\varphi_k)$, and define $\hat{T}_k(s) = (\hat G_k(s),\,\hat U_k(s))$, where $\hat U_k$ is defined in \eqref{eqn:Uk} and $\hat G_k$ is redefined (for the nonlinear case) as
$$
\hat{G}_k(s) = x_k - \eta \left( \nabla f_{I_k}(x_k) - \nabla f_{I_k}(\varphi_{k,\,I_k}) + \frac{1}{N}\sum_{n=1}^N \nabla f_n(\varphi_{k,\,n}) \right).
$$
For $b > 0$ we define
\begin{equation}\label{eq:divergence SAGA}
V_{b}\left(s_k^{(1)},\,s_k^{(2)}\right)\triangleq\|x_k^{(1)} - x_k^{(2)}\|_{2}^{2}+b\sum_{n\in [N]}\|\nabla f_{n}(\varphi_{k,\,n}^{(1)})-\nabla f_{n}(\varphi_{k,\,n}^{(2)})\|_{2}^{2}.
\end{equation}

\begin{lem}\label{lem:divergence SAGA}
The function $V_b$ in Eq. \eqref{eq:divergence SAGA} satisfies Definition \ref{def:divergence}.
\end{lem}
\begin{proof}
By definition $V_b(s,\,s') = 0$ if and only if $s = s'$. The divergence is also symmetric by definition. Choose any $q \geq 0$ and compact set $\ALP K \subset \Re^{d} \times \Re^{N\,d}$. Pick the set
$\ALP L = \{s \in \ALP S : \exists s' \in \ALP K \text{ s.t. }V_{b}(s,\,s') \leq q\}$. Since $\ALP K$ is a bounded set, $V_b$ is quadratic function of $x^{(1)}$ and $f_n$ is $c$-strongly convex for all $n\in[N]$, we conclude that $\ALP L$ is a closed and bounded set. This yields $\inf_{(s_1,s_2)\in \ALP L^\complement\times \ALP K} V_{b}(s_1,s_2)\geq q$.
\end{proof}

We select the step size $\eta > 0$ and the parameter $b > 0$ to satisfy $\eta\in\left(0,\,m/L^{2}\right)$, $\eta^{2}<b$, and $\gamma\left(\eta\right)+b\,L^{2}<1$. For contraction coefficient
\[
\alpha\left(\eta\right)=\max\left\{ \gamma\left(\eta\right)+b\,L^{2},\,\frac{\eta^{2}/b+N-1}{N}\right\},
\]
we have $\mathbb{E}[V_{b}(s_{k+1}^{(1)},\,s_{k+1}^{(2)})\, \vert \, \mathcal{F}_k] \leq \alpha(\eta)\, V_{b}(s_k^{(1)},\,s_k^{(2)})$ for all $k \geq 0$. Then, by Theorem \ref{thm:main} there is an invariant distribution $\vartheta$ such that $W_V(\mu_k,\,\vartheta) \leq \alpha(\eta)^k W_V(\mu_0,\,\vartheta)$. Assumption \ref{assu:concentration} holds and so $\vartheta = \ind{s^*}$ by Theorem \ref{thm:concentration}, where $s^* = (x^*,\,(\varphi_n^*)_{n \in [N]})$ and $\varphi_n^* = x^*$ for all $n \in [N]$.

\section{Examples in Optimization (Epoch-based Algorithms)}\label{sec:examples-epoch}

In this section, we extend our framework to epoch-based algorithms (e.g. SVRG and Catalyst) which consist of inner and outer loops. In epoch-based algorithms, each $\hat T_k$ is itself the composition of other i.i.d. random operators. We thus call $k \geq 0$ an ``epoch'' in this case (instead of an iteration), because it calls an inner loop of iterations of other random operators.  Suppose the lengths of the epochs are given by a sequence of i.i.d. stopping times $\{\tau_k\}_{k \in \Na}$ (we may simply take all $\tau_k = M\geq1$). Epoch $k\geq0$ will consist of $\tau_k$ inner iterations, starting with some function of the current iterate
$s_{k}\in\ALP S$. The inner operators are defined on a space $\tilde{\ALP S}$ (which may be different from $\ALP S$), and we let $\Pi:\tilde{\ALP S} \to\ALP S$ denote a projection operator that maps an inner iterate on $\tilde{\ALP S}$ to the original space $\ALP S$. The projection operator to be used is usually clear from the context. 
For each $m=0,\,1,\ldots,\,\tau_k-1$, define an auxiliary random operator $\hat H_{m,k}:\Omega\times\tilde{\ALP S}\times\ALP S \to \tilde{\ALP S}$, where the inner loop is the composition
\begin{equation*}
\hat{T}_{k}\left(s\right)=\Pi(\hat H_{\tau_k-1,k}\left(\cdot,\,s\right)\circ\hat H_{\tau_k-2,k}\left(\cdot,\,s\right)\circ\cdots\circ\hat H_{0,k}\left(\tilde s,\,s\right)),\,\forall k \geq 0,\label{eq:iteration-epoch}
\end{equation*}
where we have suppressed the dependence of $\hat H_{m,k}$ on $\omega$. We assume that $\{\hat H_{m,k}\}_{k\in\Na,m\in\Na}$ is a sequence of i.i.d. random operators. Thus, this construction of $\{\hat{T}_{k}\}_{k\geq 0}$ is automatically i.i.d. From starting points $s_0^{(i)}$ for $i = 1, 2$, we iteratively apply the same sequence of random operators to obtain $s_{k+1}^{(i)}=\hat{T}_{k}(s_{k}^{(i)})$ for $i = 1, 2$ for all $k \geq 0$, where
\begin{equation*}
    \tilde{s}_{m+1}^{(i)}=\hat H_{m,k}\Big(\tilde{s}_{m}^{(i)},\,s_{k}^{(i)}\Big),\,m=0,\,1,\ldots,\,\tau_k-1,
\end{equation*}
and $s_{k+1}^{(i)} = \Pi\Big(\tilde{s}_{\tau_k}^{(i)}\Big)$ for $i = 1, 2$. We next show that for various epoch-based algorithms, there is a divergence function such that the operators $\{\hat T_k\}_{k\in\Na}$ are contractions in expectation.

\subsection{The quadratic case}

\subsubsection{Stochastic variance reduced gradient descent (SVRG)}

We consider SVRG for Problem \eqref{quadratic finite sum} where $\ALP S = \tilde{\ALP S} = \Re^d$, $s_k = x_k$, and $\tilde{s}_m = \tilde{x}_m$. Within epoch $k \geq 0$ we define:
$$
\hat{H}_{m,\,k}(x) = x - \eta \left( Q_{I_m}x - Q_{I_m}x_k + \frac{1}{N} \sum_{n=1}^N (Q_n x_k + a_n)\right),
$$
where $\{I_m\}_{m \geq 0}$ is a sequence of i.i.d. uniform random variables on $[N]$. Let us define the difference sequences $\Delta x_k = x_k^{(1)} - x_k^{(2)}$ and $\Delta \tilde{x}_m = \tilde{x}_m^{(1)} - \tilde{x}_m^{(2)}$. Taking the difference of the coupled state equations gives
$\Delta \tilde{x}_{m+1} = \hat{H}_{m,\,k}(\Delta \tilde{x}_m)$, which corresponds exactly to using SVRG to solve Problem \eqref{quadratic finite sum-1}. By \cite[Theorem 1]{johnson2013accelerating}, for contraction coefficient
$\alpha(\eta) = 1/(c\,\eta (1 - 2 L \eta) N) + 2 L \eta/(1 - 2 L \eta) < 1$,
the corresponding geometric mixing rate is
$$
\mathbb{E}\left[\sum_{n=1}^N(x_k^{(1)}-x_k^{(2)})^\top Q_n (x_k^{(1)}-x_k^{(2)})\right] \leq \alpha(\eta)^k \sum_{n=1}^N(x_0^{(1)}-x_0^{(2)})^\top Q_n (x_0^{(1)}-x_0^{(2)}).
$$
Theorem \ref{thm:main} gives an invariant distribution $\vartheta$ such that $W_V(\mu_k,\,\vartheta) \leq \gamma(\eta)^k W_V(\mu_0,\,\vartheta)$. Assumption \ref{assu:concentration} holds and so $\vartheta = \ind{s^*}$ by Theorem \ref{thm:concentration}, where $s^* = x^*$.

\subsubsection{Accelerated SVRG (ASVRG)}
We consider ASVRG (see \cite{shang2018asvrg}) for Problem \eqref{quadratic finite sum} where $\ALP S = \Re^d$, $\tilde{\ALP S} = \Re^d\times\Re^d$, $s_k = x_k$, $\tilde{s}_m = (\tilde{x}_m,\,\tilde{y}_m)$, and the projection operator satisfies $\Pi(\tilde s_m) = \tilde x_m$. We then have
$\hat{H}_{m,\,k}(\tilde s_k,s_k) = x_k - \theta(\hat{L}_{m,k}(\tilde s_k) - x_k)$
where
$$
\hat{L}_{m,k}(\tilde s_k) = \tilde{y}_k - \eta \left( Q_{I_m}\tilde{x}_k - Q_{I_m}x_k + \frac{1}{N} \sum_{n=1}^N (Q_n x_k + a_n)\right).
$$
We define an additional difference sequence $\Delta \tilde{y}_{m+1} = \tilde{y}_{m+1}^{(1)} - \tilde{y}_{m+1}^{(2)}$ and then take the difference to obtain $\Delta \tilde{s}_{m+1} = \hat{H}_{m,\,k}(\Delta \tilde{s}_m)$ which corresponds to ASVRG applied to Problem \eqref{quadratic finite sum-1}. For epoch length $M \geq 1$, by \cite[Theorem 3]{shang2018asvrg} for contraction factor $\alpha(\eta,\,\theta) = 1 - \theta + \theta^2/(M\, c\, \eta)$ we have
$$
\mathbb{E}\left[\sum_{n=1}^N(x_k^{(1)}-x_k^{(2)})^\top Q_n (x_k^{(1)}-x_k^{(2)})\right] \leq \alpha(\eta,\,\theta)^k \sum_{n=1}^N(x_0^{(1)}-x_0^{(2)})^\top Q_n (x_0^{(1)}-x_0^{(2)}).
$$
Theorem \ref{thm:main} establishes an invariant distribution $\vartheta$ such that $W_V(\mu_k,\,\vartheta) \leq \gamma(\eta,\,\theta)^k W_V(\mu_0,\,\vartheta)$. Assumption \ref{assu:concentration} holds and so $\vartheta = \ind{s^*}$ by Theorem \ref{thm:concentration}, where $s^* = x^*$.

\subsection{The nonlinear case}

\subsubsection{SVRG}

We now consider SVRG-type algorithms in a separable Hilbert space $\ALP S$ (e.g. $\Re^d$ equipped with the inner product and the induced $\ell_2-$norm or any reproducing kernel Hilbert space). In this setting, in epoch $k \geq 0$ we define:
\[
\hat H_{m,k}(\tilde s_m,\,s_k) \triangleq \hat G_{m,k}\left(\tilde s_m\right)-\hat G_{m,k}\left(s_k\right)+T\left(s_k\right),
\]
where we need $\hat G_{m,k}$ to yield \textit{unbiased} estimates of the underlying contraction operator $T$ with fixed point $s^*$.
We make the following assumptions on the inner iteration within each epoch.
\begin{assumption}
\label{assu:variance-1} (i) There exists $\alpha\in\left(0,\,1\right)$
such that $\|T(s)-T (s')\|_{2}^{2}\leq\alpha\,\|s- s'\|_{2}^{2}$ for all $s,\, s' \in \ALP S$.

(ii) For all $k \geq 0$, $\mathbb E\left[\hat G_{m,k} (\tilde s)\right] = T\left(\tilde s\right)$ for all $\tilde s \in \ALP S$.

(iii) There exists $\kappa\in[0,1-\alpha)$ such that, for all $k \geq 0$,
\[
\mathbb{E}\left[\|\hat G_{m,k}\left(s\right)-\hat G_{m,k}\left(s'\right)\|_{2}^{2}\right] \leq \kappa \|s-s'\|_{2}^{2},\,\forall s,\,s' \in \ALP S.
\]
\end{assumption}

Our main convergence result for this general case follows.

\begin{theorem}\label{thm:variance-1}
[Proof in Subsection \ref{sub:svrg}] Suppose Assumption \ref{assu:variance-1} holds. Define $\xi_m \triangleq \alpha^m+\kappa(1-\alpha^m)/(1 - \alpha)$. Then, $\xi_m<1$ for any $m\in\Na$ and 
\[
\mathbb{E}\left[\|s_{k}^{(1)}-s_{k}^{(2)}\|_{2}^{2}\right] \leq \ex{\prod_{i=0}^k\xi_{\tau_k}}\|s_0^{(1)} - s_0^{(2)}\|_2^2,\quad \text{ for all } k\in\Na.
\]
\end{theorem}
\noindent
Since $\hat H_{m,k}(s^*,\,s^*) = s^*$, Theorem \ref{thm:variance-1} implies convergence of SVRG in the 2-Wasserstein metric to $\ind{s^*}$:
\beqq{W_2\left(\mu\,\mathfrak{Q}^k,\ind{s^*}\right) \leq \sqrt{ \ex{\prod_{i=0}^k\xi_{\tau_k}}}W_2\left(\mu,\ind{s^*}\right).}
Convergence in Wasserstein metric implies weak convergence, so Theorem \ref{thm:variance-1} also implies convergence of the sequence $\{s_k\}$ to $s^*$ in probability.

\subsubsection{HSAG}

Hybrid stochastic average gradient (HSAG) combines SVRG with SAGA
(see \cite{reddi2015variance}). Let $S\subset [N]$ be a subset of the objective terms to follow SAGA-type updates while the others in $S^C$ will follow SVRG-type updates with epoch lengths $\tau_{k}=M$
for all $k\geq0$. We set $\ALP S = \Re^d \times \Re^{|S|\,d}$, $\tilde{\ALP S} = \Re^d \times \Re^{N\,d}$, $s_k = (x_k,\,(\varphi_{k,\,n})_{n \in S})$, $\tilde{s}_m = (\tilde{x}_m,\, (\tilde{\varphi}_{m,\,n})_{n \in [N]})$, and $\Pi(\tilde{s}_m) = (\tilde{x}_m,\,(\tilde{\varphi}_{m,\,n})_{n \in S})$. At the beginning of epoch $k \geq 0$, we initialize $\tilde{\varphi}_{0,\,n} = x_k$ for all $n \in S^C$ (corresponding to the SVRG-type updates) and define $\hat{H}_{k,m}(s) = (\hat G_{k,m}(s),\,\hat U_{k,m}(s))$ where
$$
\hat{G}_{k,m}(\tilde{s}_m)\triangleq \tilde{x}_m - \eta \left[ \nabla f_{I_m}(\tilde{x}_m) - \nabla f_{I_m}(\tilde{\varphi}_{m,\,I_m}) + \frac{1}{N}\sum_{n=1}^N \nabla f_n(\tilde{\varphi}_{m,\,n}) \right]
$$
and
$$
\hat{U}_{k,m}(s) = \begin{cases}
\tilde{x}_m & I_m = n \in S,\\
\tilde{\varphi}_{m,\,n} & \text{otherwise}.
\end{cases}
$$

We define
\begin{equation}
V_{b,\,S}\left(s,\,s'\right)\triangleq\|x-x'\|_{2}^{2}+b\sum_{n\in S}\|\nabla f_{n}(\varphi_{k,\,n})-\nabla f_{n}(\varphi_{k,\,n}')\|_{2}^{2},
\end{equation}
which is a divergence following the same argument as for SAGA. Choose step size $\eta > 0$ and parameter $b > 0$ so that
\[
\eta\in\left(0,\,\frac{2\,c}{\left(1+|S|/N\right)L^{2}}\right),
\]
$\eta^{2}<b$, and $\gamma\left(\eta\right)+b\,|S|\,L^{2}/N<1$.
Then, for the contraction coefficient
\[
\alpha\left(\eta\right)=K\left(\eta\right)^{M}+\frac{\eta^{2}L^{2}|S^{C}|}{N\left(1-K\left(\eta\right)\right)}\left(1-K\left(\eta\right)^{M}\right),
\]
where
\[
K\left(\eta\right)=\max\left\{ \gamma\left(\eta\right)+b\,|S|\,L^{2}/N,\,\frac{\eta^{2}/b+N-1}{N}\right\},
\]
we have $\mathbb{E}[V_{b}(s_{k+1}^{(1)},\,s_{k+1}^{(2)})\, \vert \, \mathcal{F}_k] \leq \alpha(\eta)\, V_{b}(s_k^{(1)},\,s_k^{(2)})$. Theorem \ref{thm:main} gives an invariant distribution $\vartheta$ such that $W_V(\mu_k,\,\vartheta) \leq \gamma(\eta)^k W_V(\mu_0,\,\vartheta)$. Assumption \ref{assu:concentration} holds and so $\vartheta = \ind{s^*}$ by Theorem \ref{thm:concentration} where $s^* = (x^*,\,(\varphi_n^*)_{n \in S})$ and $\varphi_n^* = x^*$ for all $n \in S$.

\subsubsection{Accelerated SVRG (ASVRG)}

For ASVRG for Problem \eqref{eq:finite_sum}, we set $\ALP S = \Re^d$, $\tilde{\ALP S} = \Re^d\times\Re^d$, $s_k = x_k$, $\tilde{s}_m = (\tilde{x}_m,\,\tilde{y}_m)$, and the projection operator $\Pi(\tilde s_m) = \tilde x_m$. Then we define
$\hat{H}_{m,\,k}(s) = (x_k - \theta(\hat{L}_{m,k}(s) - x_k)$
where
$$
\hat{L}_{m,k}(s) = \arg \min_y \left\{ \langle \nabla f_{I_m}(\tilde x_m) - f_{I_m}(x_k) + \nabla f(x_k), y - \tilde{y}_m\rangle + \frac{\theta}{2 \eta} \|y - \tilde{y}_m\|_2^2 + g(y) \right\}.
$$
Let us introduce the following function:
\begin{equation}\label{eq:divergence ASVRG}
V(x,\,x') = \begin{cases}
0 & x = x',\\
\psi(x) + \psi(x') - 2\,\psi(x^*) & x \ne x',
\end{cases}
\end{equation}
based on the optimality gap. Note that $V(x,\,x') > 0$ for all $x\neq x'$ since $\psi(x) > \psi(x^*)$ for all $x \in \Re^d$, $x\neq x^*$ and $\psi$ is strongly convex.

\begin{lem}
The function $V$ in Eq. \eqref{eq:divergence ASVRG} satisfies Definition \ref{def:divergence}.
\end{lem}
\begin{proof}
(i) By definition $V(x,\,x') = 0$ if and only if $x = x'$ or $\psi(x) = \psi(x') = \psi(x^*)$, the latter of which implies $x = x' = x^*$ by strong convexity of $\psi$.

(ii) The function is symmetric by definition.

(iii) By strong convexity of $\psi$, we have $\psi(x') - \psi(x^*) \geq \frac{c}{2} \|x' - x^*\|_2^2$. Choose any $q \geq 0$ and compact set $\mathcal K \subset \Re^d$. Pick $\mathcal L=\{x': \|x' - x^*\|_2^2 \leq 2\,q/c\}\cup\ALP K$, which is a closed and bounded set and is therefore compact. This immediately yields $\inf_{(s_1,s_2)\in \ALP L^\complement\times \ALP K} V(s_1,s_2)\geq q$.
\end{proof}




For contraction factor $\alpha(\eta,\,\theta) = 1 - \theta + \theta^2/(M\, c\, \eta)$, by \cite[Theorem 3]{shang2018asvrg} we have the unconditional expectation $\mathbb E \left[V(x_{k}^{(1)},\, x_{k}^{(2)})\right] \leq \alpha(\eta)^k V(x_0^{(1)},\,x_0^{(2)})$. Theorem \ref{thm:main} gives an invariant distribution $\vartheta$ such that $W_V(\mu_k,\,\vartheta) \leq \gamma(\eta,\,\theta)^k W_V(\mu_0,\,\vartheta)$. Assumption \ref{assu:concentration} holds and so $\vartheta = \ind{s^*}$ by Theorem \ref{thm:concentration}.

\subsubsection{Catalyst}

Catalyst \cite{lin2015universal} is a time-varying algorithm. We let $\ALP S = \Re^d \times \Re^d$ and $s_k = (x_k, x_{k-1})$. Let $\theta\geq0$ be an acceleration parameter, and for each epoch $k \geq 0$ we define the objective
$$
\psi_k(x;\,s_k) \triangleq f(x) + g(x) + \frac{\theta}{2} \|x - L_k s_k\|_2^2
$$
with optimal value $\psi_k^*(s_k)$, where $L_k s_k = x_k + \beta_k(x_k - x_{k-1})$ with parameters $\{\beta_k\}_{k \geq 0}\subset(0,1)$ (here, we let $x_{-1} = x_0$ so that $L_0 s_0 = x_0$). 

In this algorithm, the inner loop which minimizes $\psi_k(\cdot;\,s_k)$ is implicit. Let $\{\epsilon_k\}_{k \geq 0}$ be a sequence of error tolerances and define
$$
\tilde{T}_{k,\,\epsilon_k}(s_k) \in \Big\{x \in \Re^d : \mathbb{E}\big[\psi_k(x; s_k) - \psi_k^*(s_k)\big] \leq \epsilon_k\Big\},
$$
so that $\hat T_k(s_k) = (\tilde{T}_{k,\,\epsilon_k}(s_k),x_k)$.
Here, $\tilde T_{k,\epsilon_k}$ returns an $\epsilon_k-$minimizer of $\psi_k(\cdot;s_k)$ in expectation. We can pick the operator $\tilde T_{k,\epsilon_k}$ corresponding to SGD, SVRG, SAGA, HSAG, etc. We obtain a sequence of time-varying operators $\{\hat{T}_k\}_{k \in \Na}$, but our convergence analysis in the Wasserstein divergence is essentially the same.

Next, we make the following parameter selections: $q = c/(c+\theta)$, $L_0 s_0 = x_0$, $\epsilon_k = (2/9)(\psi(x_0) - \psi^*)(1 - \alpha)^k$ for all $k \geq 0$, $\zeta_0 = \sqrt q$, $\zeta_k^2 = (1 - \zeta_k) \zeta_{k-1}^2 + q\, \zeta_k$ for all $k \geq 0$, and $\beta_k = \zeta_{k-1}(1 - \zeta_{k-1})/(\zeta_{k-1}^2 + \zeta_k)$ for all $k \geq 0$. Choose $\alpha < \sqrt q$ and define the divergence function as 
\begin{equation}\label{eq:divergence Catalyst}
\bar{V}(s_k,\,s_k') = V(x_k,\,x_k') + (1 - \alpha)V(x_{k-1},\,x_{k-1}'),
\end{equation}
where $V$ is defined in Eq. \eqref{eq:divergence ASVRG}. Then, by \cite[Proposition 5]{lin2017catalyst} we have
\begin{equation}\label{eq:catalyst}
\mathbb E[\bar{V}(s_k^{(1)},\,s_k^{(2)})] \leq \frac{16}{(\sqrt q - \alpha)^2}(1 - \alpha)^{k+1}V(x_0^{(1)},\,x_0^{(2)}),\,\forall k \geq 0.
\end{equation}
Eq. \eqref{eq:catalyst} is slightly different from our contraction condition
in Assumption \ref{assu:contraction}(i). By modifying the argument of Theorem \ref{thm:main} (to establish that Catalyst produces a Cauchy sequence in the Wasserstein divergence and so a limit exists, and that this limit is the same for all initial distributions), we can use Eq. \eqref{eq:catalyst} to obtain the existence of a unique invariant distribution $\vartheta$ such that
$$
W_{\bar V}(\mu_k,\,\vartheta) \leq \frac{16}{(\sqrt q - \alpha)^2}(1 - \alpha)^{k+1}V(x_0^{(1)},\,x_0^{(2)}),\,\forall k \geq 0.
$$
We can also verify that $\vartheta = \ind{s^*}$ where $s^* = (x^*,\,x^*)$.
}

\section{Functional Properties of Wasserstein Divergence}\label{sec:wass}
\new{We now turn our attention to establishing the main results of this paper. To this end, we first study some topological and functional properties of Wasserstein divergence in this section. 

The Wasserstein divergence is a generalization of the Wasserstein metric, but it enjoys many of the same properties. In particular, we show that:
\begin{enumerate}
    \item The Wasserstein divergence separates probability measures.
    \item A sequence of measures converging to a measure in the Wasserstein divergence also converges in the weak* topology (this is true for the Wasserstein metric as well). A natural consequence of this result is that the topology induced on the space of probability measures over a Polish space using the Wasserstein divergence is at least as strong as the weak* topology over the measure space.
    \item A Cauchy sequence under the Wasserstein divergence converges to a unique limit -- this is the first main result of this section. 
    \item Finally, consider two sequences $\{\mu_k\}_{k\in\Na}$ and $\{\nu_k\}_{k\in\Na}$ of probability measures that converge to limits $\theta_\mu$ and $\theta_\nu$, respectively, in the Wasserstein divergence. The second main result of this section is that if $\lim_{k\rightarrow\infty}W_V(\mu_k,\nu_k) = 0$, then $\theta_\mu = \theta_\nu$. 
\end{enumerate}
The last two results are crucial in establishing the main results of this paper.
 
}
\subsection{Properties of Wasserstein Divergence}

We start by discussing the basic properties of the Wasserstein divergence, starting with its continuity properties.



\begin{lemma}\label{lem:weakcont}
\cite[Lemma 4.3, p. 43]{villani2008optimal} The map $\xi\mapsto \int V(s_1,s_2)d\xi$ is weak* lower semi-continuous on $\wp(\mathcal S \times \mathcal S)$.
\end{lemma}

The next result exploits the structure of our divergence functions, as outlined in Definition \ref{def:divergence}, to show that the Wasserstein divergence is induced by an optimal coupling, separates distinct points, and is symmetric. These properties justify calling $W_V$ a divergence.

\begin{proposition}\label{prop:WVpproperties}
Let $V$ be a divergence function satisfying Definition \ref{def:divergence}, let $p \in [1, \infty)$, and pick $\mu_1,\mu_2 \in\wp(\ALP S)$. Then, the following statements hold:

(i) If $V$ is lower semi-continuous, then there exists an optimal coupling $\xi^*\in C(\mu_1,\mu_2)$ in the definition of $W_V$. Consequently, there exists a pair of random variables $(s_1,\,s_2)$ with distribution $\xi^*$ such that
\begin{equation}
\mathbb E\left[V(s_1,\,s_2)\right]=W_V\left(\mu_{1},\,\mu_{2}\right).\label{eqn:vs1s2}
\end{equation}

(ii) If $V$ is positive definite, then $W_V\left(\mu_{1},\,\mu_{2}\right) = 0$ if and only if $\mu_1 = \mu_2$.

(iii) If $V$ is symmetric, then $W_V\left(\mu_{1},\,\mu_{2}\right) = W_V\left(\mu_{2},\,\mu_{1}\right)$.

(iv) If $V = \rho$, then $W_V$ coincides with $W_1$ (the usual 1-Wasserstein distance), and is a metric. Furthermore, the space of probability measures is complete under this metric.
\end{proposition}
\begin{proof}
The first result follows because $C(\mu_1,\mu_2)$ is a weak* compact set in the space of probability measures (\cite[Lemma 4.4, p. 44]{villani2008optimal}) and by Lemma \ref{lem:weakcont}. Consequently, we can apply the Weierstrass extreme value theorem to demonstrate the existence of $\xi^*$ that achieves the infimum in the definition of $W_V$ (which is the desired optimal coupling). Since an optimal coupling $\xi^*$ exists, by Strassen's theorem \cite{strassen1965existence}, there exists a pair of random variables $(s_1,s_2)$ on $\ALP S$ which has the joint distribution $\xi^*$. For this pair, the equality in \eqref{eqn:vs1s2} holds.

For the second statement, if $\mu_1 = \mu_2$, by Strassen's theorem \cite{strassen1965existence}, we can define random variables $s_1,s_2$ such that $s_1 = s_2$ almost surely. As a result, $W_V\left(\mu_{1},\,\mu_{2}\right) = 0$. To prove the converse, if $W_V\left(\mu_{1},\,\mu_{2}\right) = \inf_{\xi\in C(\mu_1,\mu_2)}\ex{V(s_1,s_2)} = 0$, then it must be that $V(s_1,s_2) = 0$, $\xi^*$-almost surely. Consequently, we must have $\mu_1 = \mu_2$ since $\xi^* \in C(\mu_1,\mu_2)$.

We now establish the third statement. If $V$ is symmetric, then 
\beqq{W_V\left(\mu_{1},\,\mu_{2}\right) = \inf_{\xi\in C(\mu_1,\mu_2)}\ex{V(s_1,s_2)} = \inf_{\xi\in C(\mu_2,\mu_1)}\ex{V(s_2,s_1)} = W_V\left(\mu_{2},\,\mu_{1}\right),}
where the second equality follows from symmetry of $V$. The fourth statement follows from \cite[Theorem 4.3]{villani2008optimal}.
\end{proof}

The next two results together establish that if a sequence of probability measures converges to another measure in $W_V$, then the limit point is also a probability measure, and the sequence converges to this probability measure in the weak* topology. This is a well-known result for the Wasserstein metric (note that the topology induced by the Wasserstein metric is stronger than the weak* topology).  We first need the following lemma.

\begin{lemma}\label{lem:xikxi}
Let $\{\mu_k\}_{k\in\Na}$ be a sequence of probability measures converging to $\theta$ in the weak* topology. Pick any $\nu\in\wp(\ALP S)$ and let $\xi_k^* \in C(\mu_k,\nu)$ be an optimal coupling in $W_V(\mu_k,\nu)$ for all $k \in \Na$. Then, the set of probability measures $\{\xi_k^*\}_{k\in\Na}$ is tight and there exists a subsequence $\{\xi_{k_l}^*\}_{l\in\Na}$ that converges to $\xi\in C(\theta,\nu)$.
\end{lemma}
\begin{proof}
 Since $\{\mu_k\}_{k\in\Na}$ converges to $\theta$, the set $\{\mu_k\}_{k\in\Na} \cup \{\theta\}$ is tight. Consequently, for any $\epsilon>0$ there is a compact set $\ALP K\subset\ALP S$ such that $\mu_k(\ALP K)\geq 1-\epsilon$ for all $k\in\Na$ and $\theta(\ALP K)\geq 1-\epsilon$. Let $\ALP L\subset\ALP S$ be another compact set such that $\nu(\ALP L)\geq 1-\epsilon$. Then, for all $k\in\Na$, $\xi_k^*$ satisfies
 \beqq{\xi_k^*((\ALP K \times\ALP L)^\complement) = \xi_k^*(\ALP K^\complement \times\ALP L) + \xi_k^*(\ALP S\times\ALP L^\complement) \leq \mu_k(\ALP K^\complement)+\nu(\ALP L^\complement)<2\epsilon,}
 since $\xi_k^* \in C(\mu_k,\nu)$. As a result, the set of probability measures $\{\xi_k^*\}_{k\in\Na}$ is tight. Now pick a convergent subsequence $\{\xi_{k_l}^*\}_{l\in\Na}$, and let $\xi$ denote its weak* limit. 
 
Recall that the push forward of a measure through the projection mapping is continuous \cite[Section 5.2]{ambrosio2008gradient}. Thus, the marginals of $\xi_{k_l}^*$ are $\mu_{k_l}$ and $\nu$, $\mu_{k_l}$ converges to $\theta$, and $\nu$ trivially converges to $\nu$. We conclude that the marginals of $\xi$ are $\theta$ and $\nu$, which implies $\xi\in C(\theta,\nu)$ as desired. 
\end{proof}

The preceding lemma leads to the following result.

\begin{theorem}\label{thm:divergencelimit}
Let $V$ be a divergence function satisfying Definition \ref{def:divergence}. Let $\{\mu_k\}_{k\in\Na}$ be a sequence of probability measures and let $\theta$ be another measure (not necessarily a probability measure) such that $\lim\sup_{k\rightarrow \infty} W_V(\mu_k,\theta)=0$. Then, $\theta$ is also a probability measure and $\mu_k$ converges to $\theta$ in the weak* topology. 
\end{theorem}
\begin{proof}
Choose $K\in\Na$ such that $W_V(\mu_k,\theta)\leq 1$ for all $k\geq K$. Since $V$ satisfies condition (iii) in Definition \ref{def:divergence}, by \cite[Lemma 7.13, p. 107]{gupta2014phd} we know that the set $\{\mu_k\}_{k\geq K}$ is a weak* precompact set of probability measures. Consequently, there exists a weak* convergent subsequence $\{\mu_{k_l}\}_{l\in\Na}$ converging to a weak* limit, say $\tilde\theta$, which is a probability measure. Since $\lim\sup_{k\rightarrow \infty} W_V(\mu_k,\theta)=0$, we conclude that $\lim\sup_{l\rightarrow \infty} W_V(\mu_{k_l},\theta)=0$. We next establish that $\tilde\theta = \theta$.

Let $\xi_k^* \in C(\mu_k,\theta)$ be an optimal coupling in $W_V(\mu_k,\theta)$. Since $W_V(\mu_{k_l},\theta) \rightarrow 0$ and $\mu_{k_l}\rightarrow\tilde\theta$ in the weak* topology, by Lemma \ref{lem:xikxi} we conclude that there exists a further subsequence, which we denote by $\xi_{k_l}^*$ by a slight abuse of notation, that converges to a limit $\tilde \xi^* \in C(\tilde\theta,\theta)$. Since $V$ is lower semi-continuous (but possibly unbounded), Lemma \ref{lem:weakcont} shows that
\beqq{\int V(s_1,s_2) \tilde\xi^*(ds_1,ds_2)\leq \underset{l\rightarrow\infty}{\lim\inf}\int V(s_1,s_2) \xi_{k_l}^*(ds_1,ds_2) = 0\implies W_V(\tilde\theta,\theta) = 0.}
Since $V$ is positive definite, we conclude that $\tilde\theta = \theta$ by Proposition \ref{prop:WVpproperties}. 

The above argument also implies that $\{\mu_k\}_{k\in\Na}$ is tight and that the set of weak* limit points of $\{\mu_k\}_{k\in\Na}$ is the singleton $\{\theta\}$. Since there is a unique limit point of the sequence $\{\mu_k\}_{k\in\Na}$, the entire sequence must converge to $\theta$, and hence the proof is complete.
\end{proof}

\begin{example}
We show by example that $V$ must satisfy condition (iii) in Definition \ref{def:divergence} for the above result to hold. Let $\ALP S = \Re$ and let $V(s_1,s_2) = |s_1-s_2|\exp(-|s_1-s_2|)$. It is clear that $V$ is a positive definite divergence function. Pick $\mu_k = \ind{k}$ and $\theta = 0$ (the zero measure). We readily have $W_V(\mu_k,\theta) = k\exp(-k) \rightarrow 0$ as $k\rightarrow\infty$. In other words, the Wasserstein divergence limit of $\mu_k$ is not a probability measure. The above theorem shows that this situation will not arise if $V$ satisfies a growth condition like condition (iii) in Definition \ref{def:divergence}.
\end{example}

\subsection{Convergence of Cauchy Sequences under Wasserstein Divergence}
It is well-known that a Cauchy sequence in a large class of metric spaces converges (the class of complete metric spaces \cite{ali2006}). We can readily adapt the definition of a Cauchy sequence to a space endowed with a divergence.

\begin{defn}
A sequence of measures $\{\mu_k\}_{k\in\Na}\subset\ALP P_V(\ALP S)$ is said to be a Cauchy sequence under the Wasserstein divergence if and only if for every $\epsilon>0$, there exists $K_\epsilon\in\Na$ such that $W_V(\mu_k,\mu_{k+l}) < \epsilon$ for all $k\geq K_\epsilon$ and $l\in\Na$.
\end{defn}

We next show that a Cauchy sequence under the Wasserstein divergence converges. This is the first key property that we need for our proof of Theorem \ref{thm:main}.

\begin{proposition}
\label{prop:WVplimitexistence}
Let $V$ be a divergence function satisfying Definition \ref{def:divergence} and $\{\mu_k\}_{k\in\Na}\subset\ALP P_V(\ALP S)$ be a Cauchy sequence under the Wasserstein divergence. Then, there exists a probability measure $\theta$ such that $\mu_k\rightarrow\theta$ in the Wasserstein divergence. Further, $\mu_k\to\theta$ in the weak* sense.
\end{proposition}
\begin{proof}
Pick $\epsilon>0$ and consider the set of measures $\ALP M_{\epsilon} = \{\nu\in\wp(\ALP S):W_V(\nu,\mu_{K_\epsilon}) \leq \epsilon\}$. Since $V$ is lower semicontinuous and satisfies condition (iii) in Definition \ref{def:divergence}, by \cite[Lemma 7.13, p. 107]{gupta2014phd}, the set $\ALP M$ is a weak* compact set of measures. Now, note that $\{\mu_k\}_{k\geq K_\epsilon}\subset\ALP M$, which implies that $\{\mu_k\}_{k\geq K_\epsilon}$ is a weak* precompact set of measures. Thus, there exists a weak* convergent subsequence $\{\mu_{k_m}\}$ such that $\mu_{k_m}\rightarrow\theta$ in the weak* sense, where $\theta$ is a probability measure. We claim that $\lim_{k\rightarrow\infty} W_V(\theta,\mu_k) = 0$, which we prove below. 

Pick $l\in\Na$. Let $\xi_{kl}^*$ be an optimal coupling in $W_V(\mu_k,\mu_{K_\epsilon+l})$. Since $\mu_{k_m}$ converges to $\theta$ in the weak* sense, we use Lemma \ref{lem:xikxi} to conclude that there exists a subsequence $\xi_{k_{m_n}l}^*$ that converges to some $\xi_l^*\in C(\theta,\mu_{K_\epsilon+l})$. Using Lemma \ref{lem:weakcont}, we conclude that
\beqq{W_V(\theta,\mu_{K_\epsilon+l}) \leq \int Vd\xi_l^*\leq \underset{n\rightarrow \infty}{\lim\inf}\; \int Vd\xi_{k_{m_n}}^* =  \underset{n\rightarrow \infty}{\lim\inf}\; W_V(\mu_{k_{m_n}},\mu_{K_\epsilon+l})<\epsilon.}
From Proposition \ref{prop:WVpproperties}(iii), we know that $W_V(\theta,\mu_{K_\epsilon+l}) = W_V(\mu_{K_\epsilon+l},\theta)$. Thus, for every $\epsilon>0$, there exists $K_\epsilon\in\Na$ such that $W_V(\mu_k,\theta)<\epsilon$ for all $k\geq K_\epsilon$. Since $\epsilon > 0$ is arbitrary, we must have $\lim_{k\rightarrow\infty} W_V(\mu_k,\theta)=0$. By Theorem \ref{thm:divergencelimit}, we conclude that $\mu_k$ converges to $\theta$ in the weak* sense.
\end{proof}

\noindent
In the preceding proof, we invoked symmetry of the Wasserstein divergence for the first time. In some cases, the divergence function $V$ may not be symmetric, however one could potentially verify that the above result still holds by exploiting the specific structure of the divergence function in question.

We now consider two sequences of measures that converge with respect to the Wasserstein divergence. If the divergence between the elements of these two sequences approaches zero, then it is reasonable to expect that the limits of the two sequences should be the same. This property of the Wasserstein divergence is established in the next proposition, and it is the second key result leading to the proof of Theorem \ref{thm:main}.
\begin{proposition}
\label{prop:WVplimitsequal}
Let $V$ be a divergence function satisfying Definition \ref{def:divergence}. Let $\{\mu_k\}_{k\in\Na}$ and $\{\nu_k\}_{k\in\Na}$ be two convergent sequences of probability measures converging to $\theta_\mu$ and $\theta_\nu$, respectively, in the Wasserstein divergence. Then,
\beqq{\lim_{k\rightarrow\infty}W_V(\mu_k,\nu_k) = 0\implies\theta_\mu = \theta_\nu.}
\end{proposition}
\begin{proof}
From Theorem \ref{thm:divergencelimit}, we conclude that $\{\mu_k\}_{k \in \Na}$ and $\{\nu_k\}_{k \in \Na}$ converge to $\theta_\mu$ and $\theta_\nu$ in the weak* topology, respectively. Let $\xi_k^*$ be an optimal coupling in $W_V(\mu_k,\nu_k)$. By essentially the same argument as in Lemma \ref{lem:xikxi}, we conclude that $\{\xi_k^*\}_{k\in\Na}$ is a tight set of measures. Therefore, it includes a convergent subsequence $\{\xi_{k_l}^*\}_{l\in\Na}$ which converges to some $\xi\in C(\theta_\mu,\theta_\nu)$ in the weak* sense. Further, this coupling $\xi$ satisfies
\beqq{W_V(\theta_\mu,\theta_\nu) \leq \int Vd\xi \leq \underset{n\rightarrow \infty}{\lim\inf}\; \int Vd\xi_{k_l}^* =  \underset{n\rightarrow \infty}{\lim\inf}\; W_V(\mu_{k_l},\nu_{k_l}) = 0, }
by Lemma \ref{lem:weakcont}. By Proposition \ref{prop:WVpproperties}(i), the above expression immediately yields $\theta_\mu = \theta_\nu$, completing the proof.
\end{proof}

\subsection{Proofs of the Main Results}

\new{
\subsubsection{Proof of Theorem \ref{thm:conditional_expectation}}\label{sub:conditional}
(i) There exist random variables $s,\,s'$, independent of $\{\hat{T}_k\} _{k \in \Na}$, such that $\mathbb{E}\left[V\left(s,\,s'\right)\right]=W_V\left(\mu_{1},\,\mu_{2}\right)$, by Proposition \ref{prop:WVpproperties}(i). Then, we have
\[
W_V\left(\mu_{1}\mathfrak{Q},\,\mu_{2}\mathfrak{Q}\right)\leq\,\mathbb{E}\left[V\left(\hat{T}_{0}\left(s\right),\,\hat{T}_{0}\left(s'\right)\right)\right] \leq\, \alpha\,\mathbb{E}\left[V\left(s,\,s'\right)\right]
=\, W_V\left(\mu_{1},\,\mu_{2}\right),
\]
where the first inequality follows by definition of $W_V\left(\mu_{1}\mathfrak{Q},\,\mu_{2}\mathfrak{Q}\right)$
and the second inequality follows from the hypothesis.

(ii) The inequality $W_V(\mu,\mu\circ f^{-1})\leq \ex{V(s,f(s))}$ holds since $W_V$ is the infimum of $\int Vd\xi$ over all couplings $\xi \in C(\mu, \mu\circ f^{-1})$, and the coupling induced by $(s,\,f(s))$ is just one such coupling in $C(\mu, \mu\circ f^{-1})$. 

(iii) We establish this part with the following lemma.

\begin{lemma}\label{lem:couplingsingle}
For any $\mu\in\wp(\ALP S)$, we have $C(\mu,\ind{s^*}) = \{\mu\, \ind{s^*}\}$, i.e., there is only one coupling between $\mu$ and $\ind{s^*}$.
\end{lemma}
\begin{proof}
Any coupling $\xi\in C(\mu,\ind{s^*})$ can be distintegrated as
\[
\xi(ds_1,ds_2) = \xi(ds_1|s_2) \xi(ds_2) = \xi(ds_1|s_2) \ind{s^*}(ds_2).
\]
The marginal measure on $\ALP S$ is $\mu(ds_1) = \xi(ds_1|s^*)$. Thus, any coupling $\xi\in C(\mu,\ind{s^*})$ must satisfy $\xi(\cdot|s^*) = \mu(\cdot)$. Two couplings $\xi,\tilde\xi\in C(\mu,\ind{s^*})$ coincide if $\xi(\cdot|s_2)$ and $\tilde\xi(\cdot|s_2)$ differ only on a set of $\ind{s^*}$-measure zero. As a result, any $\xi\in  C(\mu,\ind{s^*})$ coincides with $\tilde \xi = \mu\, \ind{s^*}$. 
\end{proof}

}

\subsubsection{Proof of Theorem \ref{thm:main}}\label{sub:main}

(i) This part follows immediately by iterating the recursion in Assumption \ref{assu:contraction}(i).

(ii) Choose any $\mu\in\ALP M$ and define the sequence $\mu_k = \mu\,\mathfrak Q^k$ for all $k \in \Na$. Let $c:=\sup_{l\geq 0}W_V(\mu,\mu_l)$. Now pick any $l\in\Na$, we then have $W_V\left(\mu_k,\,\mu_{k+l}\right) = W_V\left(\mu\,\mathfrak Q^k,\,\mu\,\mathfrak Q^{k+l}\right) \leq \alpha^k W_V\left(\mu,\,\mu\,\mathfrak Q^l\right)\leq c \alpha^k$ by Assumption \ref{assu:contraction}(ii). Consequently, we can apply Proposition \ref{prop:WVplimitexistence} to conclude that there exists a probability measure $\vartheta$ such that $\lim_{k\rightarrow\infty} W_V\left(\mu_k,\,\vartheta\right)=0$.

Now we show that this limit is the same for all initial $\mu\in\ALP M$. Choose any $\mu,\,\nu \in \ALP M$ and define the sequences $\mu_k = \mu\,\mathfrak Q^k$ and $\nu_k = \nu\,\mathfrak Q^k$ for all $k \in \Na$. By the previous argument, $\mu_k$ converges to some $\vartheta_\mu$, and $\nu_k$ converges to some $\vartheta_\nu$ (also in the weak* sense). Furthermore,
\[
W_V\left(\mu_k,\,\nu_k\right) = W_V\left(\mu\,\mathfrak Q^k,\,\nu\,\mathfrak Q^k\right) \leq \alpha^k W_V\left(\mu,\,\nu \right).
\]
By Proposition \ref{prop:WVplimitsequal}, it follows that $\vartheta_\mu = \vartheta_\nu$ and so the limit must be the same for all initial $\mu\in\ALP M$. We denote this limit as $\vartheta$. Note that this limit is unique since, for any initial condition, the limiting measure is always $\vartheta$.

To complete the proof, we show that $\vartheta$ is invariant with respect to $\mathfrak Q$. As before, choose any $\mu\in\ALP M$ and define the sequence $\mu_k = \mu\,\mathfrak Q^k$ for all $k \in \Na$. We have just shown that $\lim_{k \rightarrow \infty}W_V\left(\mu_k,\,\vartheta\right)=0$. Since $W_V\left(\mu_k\,\mathfrak Q,\,\vartheta\,\mathfrak Q\right) \leq \alpha W_V\left(\mu_k,\,\vartheta\right)$, we see that $\lim_{k \rightarrow \infty}W_V\left(\mu_k\,\mathfrak Q,\,\vartheta\,\mathfrak Q\right) = 0$ and so the sequence $\{\mu_k \mathfrak Q\}_{k \in \Na}$ must converge to $\vartheta\,\mathfrak Q$. By Proposition \ref{prop:WVplimitsequal}, the limits of the sequences $\{\mu_k\}_{k \in \Na}$ and $\{\mu_k \mathfrak Q\}_{k \in \Na}$ are equal, and thus $\vartheta = \vartheta\,\mathfrak Q$.

(iii) Follows from invariance of $\vartheta$ with respect to $\mathfrak Q$, i.e., $\vartheta = \vartheta\,\mathfrak Q^k$ for all $k \in \Na$.

(iv) This is merely a consequence of Proposition \ref{prop:WVplimitexistence}.

\subsubsection{Proof of Theorem \ref{thm:variance-1}}\label{sub:svrg}

Let $\tilde{\ALP F}_{m,k}$ denote the $\sigma$-algebra generated by the random variables $(s_1,\ldots,s_k,\tilde s_1,\ldots,\tilde s_m)$. As a consequence of Assumption \ref{assu:variance-1} (by expanding the squared-norm and using conditional unbiasedness), we have:
\[
\mathbb{E}\left[\|\tilde s_{m+1}^{(1)}- \tilde s_{m+1}^{(2)}\|_{2}^{2}\,\vert\,\tilde{\mathcal F}_{m,k}\right]\leq\alpha\,\|\tilde s_{m}^{(1)}-\tilde s_{m}^{(2)}\|_{2}^{2}+\kappa \|s_k^{(1)}-s_k^{(2)}\|_{2}^{2},\,\forall m \geq 0.
\]
Recall that we initialize epoch $k \geq 0$ with $\tilde s_0^{(i)} = s_k^{(i)}$ for $i = 1, 2$, then we get
\[
\mathbb{E}\left[\|\tilde s_{m+1}^{(1)} - \tilde s_{m+1}^{(2)}\|_{2}^{2}\,\vert\,\tilde{\mathcal F}_{m,k}\right]\leq \xi_m \|s_k^{(1)}-s_k^{(2)}\|_{2}^{2},\,\forall m \geq 0.
\]
The desired result then follows by iterating this recursion. Taking $s_k^{(2)} = s^*$ for all $k \geq 0$ (which holds by construction of the variance-reduced operator if we initialize with $s_0^{(2)} = s^*$) shows that the invariant distribution is $\ind{s^*}$.

\new{
\section{Discussions and Open Problems}\label{sec:open}
We propose the new notion of Wasserstein divergence which generalizes the Wasserstein distance. Divergences on the space of probability measures are central in information theory and statistics. Examples of various divergences include the Kullback-Leibler divergence, $\chi^2$ divergence, Bregman divergence, $f$-divergence, Sinkhorn divergence, etc. One of the main drawbacks of these divergences is that they require certain absolute continuity conditions to hold to compute the divergence between two measures. Thus, these divergences are not suitable for studying RSAs in optimization, machine learning, and reinforcement learning, since the absolute continuity condition may not hold in general. In contrast, the Wasserstein divergence does not require an absolute continuity condition.

We develop a theory for convergence of constant step-size RSAs with respect to the Wasserstein divergence. We show that operators which are contractions with respect to the Wasserstein divergence enjoy certain stability properties: (i) they have an invariant distribution, and (ii) they converge to this distribution at a geometric rate. In cases where the random operator maps the optimal solution back to itself (e.g. variance reduced algorithms), we show that iteration of random operators converges to the optimal solution in probability.

Several examples in optimization demonstrate the usefulness of the Wasserstein divergence in analyzing the stability of RSAs. Our framework is general enough to be applied to other situations where there is some underlying contractive property, and also allows the underlying spaces over which the random operators to be Polish spaces, e.g. MDPs, continuous state/action MDPs with function fitting, zero-sum minimax MDPs, and dynamical systems.

This paper also provides a new topology on the space of measures $\tau_V$ generated on $\ALP P_V(\ALP S)$ by the open sets of the form $\{\nu\in\ALP P_V(\ALP S): W_V(\mu,\nu)<1/k\}_{k\in\Na}$. In Theorem \ref{thm:divergencelimit}, we have shown that the topology $\tau_V$ is finer than $\tau_{w*}$, the weak* topology on $\ALP P_V(\ALP S)$. Convergence of a sequence of measures in the Wasserstein divergence also differs from convergence in other divergences. In this way, the Wasserstein divergence is a useful new notion to study convergence of sequence of measures.

Our present paper creates the tools to address the following extensions and open problems:
\begin{open}
In practice, we would like to run an RSA with a fixed step-size to quickly reach a neighborhood of the desired solution, and then switch to a variable decreasing step-size to ensure convergence of the algorithm. Our framework can provide supporting theory to determine when the marginal distribution of the RSA is close enough to the invariant distribution to switch from fixed to variable step-sizes.
\end{open}

\begin{open}
We used a fixed divergence function $V$ in this analysis. We would like to extend our results to time-varying divergence functions $\{V_k\}_{k \geq 0}$ to support non-strongly convex problems, variable step-size RSAs and RSAs with weighted averaging of the iterates, and mirror-descent type algorithms.
\end{open}

\begin{open}
In unbiased variance reduced algorithms (e.g. SVRG, SAGA, and HSAG), the expected difference $\|x_{k+1}^{(1)} - x_{k+1}^{(2)}\|_2^2$ is upper bounded by the sum of two terms: (i) a geometric factor of the previous difference and (ii) the expected squared norm of the difference between the errors in the gradient estimation. We would like to cast biased variance reduced algorithms such as SAG, B-SAGA, and B-SVRG (see \cite{driggs2019biased}) in our framework using more sophisticated upper bounds.
\end{open}

\begin{open}
  In \cite{munos2008finite}, the authors derive finite time $\ALP L_p-$bounds error bounds for sampling based fitted value iteration for continuous-state finite-action MDPs which are similar to Theorem \ref{thm:contractionerror}. The authors make some strong assumptions on the mixing properties of the transition kernel of the MDP to derive these bounds. We would like to relax these assumptions and construct a divergence for which we can obtain finite time error bounds with respect to the corresponding Wasserstein divergence.
\end{open}    
}

\bibliographystyle{siamplain}
\bibliography{References/edp,References/guptaconf,References/guptajournal,References/haskelljournal,References/prob,References/probbook,References/ql,References/sa,References/sgd}

\newpage
\appendix

\section*{Appendix A: Supplement for Section \ref{sec:examples} (optimization algorithms)}

\subsection*{Details for oracle-based SGD (the quadratic case)}

For $T(x) = x - \eta\, Q\,x$, we have

\begin{align*}
 \|T(x_1) - T(x_2)\|_{2}^{2} =\, & \|x_1 - \eta\, Q\,x_1 - \left(x_2 - \eta\, Q\,x_2 \right)\|_2^2 \\
 =\, & \|x_1 - x_2\|_2^2 - 2\, \eta \langle x_1 - x_2,\,Q\,x_1 - Q\,x_2 \rangle +\eta^2\|Q\,x_1 - Q\,x_2\|_2^2 \\
\leq\, &  \gamma(\eta) \|\Delta x_{k}\|_2^2,
\end{align*}
using $\langle x_1 - x_2,\,Q\,x_1 - Q\,x_2 \rangle \geq c\,\|x_1 - x_2\|_2^2$ and $\|Q\,x_1 - Q\,x_2\|_2^2 \leq L^2 \|x_1 - x_2\|_2^2$ by assumption that $c\, I_d \preceq Q \preceq L\, I_d$ and $0 < c \leq L$.

\subsection*{Details for ASGD (the quadratic case)}

We take the coupling $\varepsilon_{k}^{(1)} = \varepsilon_{k}^{(2)} = \varepsilon_k$ for all $k \geq 0$ to obtain:
\begin{align}
    y_k^{(i)} =\, & (1+\alpha)x_k^{(i)} - \alpha\, x_{k-1}^{(i)},\\
    x_{k+1}^{(i)} =\, & (1 + \beta)x_k^{(i)} - \beta\, x_{k-1}^{(i)} - \eta\, [\nabla f(y_k^{(i)}) + \varepsilon_{k+1}],
\end{align}
for $i = 1,\,2$. We define the difference sequence
$$
\Delta s_k = (s_k^{(1)} - s_k^{(2)})^\top = ((x_k^{(1)} - x_k^{(2)})^\top,\, (x_{k-1}^{(1)} - x_{k-1}^{(2)})^\top).
$$
The corresponding linear dynamical system is
$$
\Delta s_{k+1} = A\, \Delta s_{k} + B\, w_k,\,\forall k \geq 0,
$$
for matrices $A = \tilde A \otimes I_d$ and $B = \tilde B \otimes I_d$ where
\[
\tilde A = \left[\begin{array}{cc}
1 + \beta & -\beta\\
1 & 0
\end{array}\right] \text{ and }
\tilde B = \left[\begin{array}{c}
-\alpha\\
0
\end{array}\right],
\]
and with noise
$$
w_k = \nabla f(y_k^{(1)}) - \nabla f(y_k^{(2)}) = \nabla f((1+\alpha)x_k^{(1)} - \alpha\, x_{k-1}^{(1)}) - \nabla f((1+\alpha)x_k^{(2)} - \alpha\, x_{k-1}^{(2)}).
$$

For the quadratic case, we have the following functional inequalities. First, we have

\begin{align*}
 & f\left(\Delta y_{k}\right)-f\left(\Delta x_{k+1}\right)\\
=\, & f\left(\Delta y_{k}\right)-f\bigg(\left(1+\beta\right)x_{k}^{\left(1\right)}-\beta\,x_{k}^{\left(1\right)}-\eta\,\nabla f\left(y_{k}^{\left(1\right)}\right)\\
& \hspace{10em}-\left(\left(1+\beta\right)x_{k}^{\left(2\right)}-\beta\,x_{k}^{\left(2\right)}-\eta\,\nabla f\left(y_{k}^{\left(2\right)}\right)\right)\bigg)\\
\geq\, & \nabla f\left(\Delta y_{k}\right)^{\top}\left(\left(\beta-\alpha\right)\left(\left(\Delta x_{k-1}\right)-\left(\Delta x_{k}\right)\right)+\eta\left(\nabla f\left(y_{k}^{\left(1\right)}\right)-\nabla f\left(y_{k}^{\left(2\right)}\right)\right)\right)\\
 & -\frac{L}{2}\|\left(\beta-\alpha\right)\left(\left(\Delta x_{k-1}\right)-\left(\Delta x_{k}\right)\right)+\eta\left(\nabla f\left(y_{k}^{\left(1\right)}\right)-\nabla f\left(y_{k}^{\left(2\right)}\right)\right)\|_{2}^{2}.
\end{align*}
Second, we have
\begin{align*}
 f\left(\Delta x_{k}\right)-f\left(\Delta y_{k}\right) \geq\, & \nabla f\left(\Delta y_{k}\right)^{\top}\left(\left(\Delta x_{k}\right)-\left(\Delta y_{k}\right)\right)+\frac{c}{2}\|\left(\Delta x_{k}\right)-\left(\Delta y_{k}\right)\|_{2}^{2}\\
=\, & \alpha\nabla f\left(\Delta y_{k}\right)^{\top}\left(\left(\Delta x_{k-1}\right)-\left(\Delta x_{k}\right)\right) +\frac{c\,\alpha^{2}}{2}\|\left(\Delta x_{k-1}\right)-\left(\Delta x_{k}\right)\|_{2}^{2}.
\end{align*}
Finally, we have
\begin{align*}
f\left(0\right)-f\left(\Delta y_{k}\right) \geq\, & \nabla f\left(\Delta y_{k}\right)^{\top}\left(-\left(\Delta y_{k}\right)\right)+\frac{c}{2}\|-\left(\Delta y_{k}\right)\|_{2}^{2}\\
=\, & -\nabla f\left(\Delta y_{k}\right)^{\top}\left(\left(1+\alpha\right)\left(\Delta x_{k}\right)-\alpha\left(\Delta x_{k-1}\right)\right)+\\
&\frac{c}{2}\|\left(1+\alpha\right)\left(\Delta x_{k}\right)-\alpha\left(\Delta x_{k-1}\right)\|_{2}^{2},
\end{align*}
using
$\Delta y_{k}=\left(1+\alpha\right)\left(\Delta x_{k}\right)-\alpha\left(\Delta x_{k-1}\right)$.
Adding the first and second inequalities gives a lower bound on
$f\left(\Delta x_{k}\right)-f\left(\Delta x_{k+1}\right)$,
and adding the first and third inequalities gives a lower bound on
$f\left(0\right)-f\left(\Delta x_{k+1}\right)$.

Next, define the matrix $X = X_1 + \rho^2 X_2 + (1 - \rho^2) X_3 \in \mathbb{R}^{2d \times 2d}$ as in \cite[Lemma 5]{hu2018dissipativity} so that
\begin{equation*}
S(\Delta s_{k},\,w_k) \triangleq \left[ \begin{array}{c} \Delta s \\ w \end{array} \right]^\top X \left[ \begin{array}{c} \Delta s \\ w \end{array} \right] \leq - (f(x_{k+1}^{(1)} - x_{k+1}^{(2)}) - f(0)) + \rho (f(x_{k}^{(1)} - x_{k}^{(2)}) - f(0)).
\end{equation*}
We need $\rho = \rho_{\alpha,\,\beta} \in (0,\,1)$ such that the LMI
$$
\left(\begin{array}{cc}
A^\top P\, A - \rho\,P & A^\top P\, B\\
B^\top P\, A & B^\top P\, B
\end{array}\right) - X \preceq 0,
$$
is satisfied for some $P \in \mathbb{S}_+^d$. It then follows that
$$
\Delta s_{k+1}^\top P_{\alpha,\,\beta} \Delta s_{k+1} - \rho\, \Delta s_{k}^\top P_{\alpha,\,\beta} \Delta s_{k} \leq S(\Delta s_{k},\,w_k),
$$
and we arrive at the desired result
$$
\Delta s_{k+1}^\top P_{\alpha,\,\beta} \Delta s_{k+1} + f(x_{k+1}^{(1)} - x_{k+1}^{(2)}) \leq \rho \left[ \Delta s_{k}^\top P_{\alpha,\,\beta} \Delta s_{k} + f(x_{k}^{(1)} - x_{k}^{(2)}) \right].
$$

\subsection*{Details for SGD}

We take the coupling $I_k^{(1)} = I_k^{(2)} = I_k$ for all $k\geq0$ to obtain:

\begin{equation}\label{eq:SGD}
    x_{k+1}^{(i)} = x_k^{(i)} - \eta\, \nabla f_{I_k}x_k^{(i)},\,\forall k \geq 0,
\end{equation}
for $i = 1, 2$. We then have (in the almost sure sense) that
\begin{equation*}
 \|\Delta x_{k+1}\|_{2}^{2} \leq \|x_{k}^{\left(1\right)}-\eta\,\nabla f_{I_k}(x_{k}^{\left(1\right)}) - \left(x_{k}^{\left(2\right)}-\eta\,\nabla f_{I_k}(x_{k}^{\left(2\right)}) \right)\|_2^2
\leq  \gamma(\eta) \|\Delta x_{k}\|_2^2,
\end{equation*}
where the first inequality follows by non-expansiveness of the projection operator, and the second is by the fact that $I - \eta\, \nabla f_{n}$ is a $\gamma(\eta)-$contraction for all $n \in [N]$. The argument for batch sampling is similar, using the fact that $I - (\eta/J)\sum_{j=1}^J \nabla f_{I_j}$ is a $\gamma(\eta)-$contraction for any subset $\{I_1,\ldots,\,I_J\} \subset [N]$.

\subsection*{Details for SAGA}

We take the coupling $I_k^{(1)} = I_k^{(2)} = I_k$ for all $k\geq0$ to obtain:
\begin{align}\label{eq:SAGA quadratic}
    x_{k+1}^{(i)} = & x_k^{(i)} - \eta \left( \nabla f_{I_k}x_k^{(i)} - \nabla f_{I_k}\varphi_{k,\,I_k}^{(i)} + \frac{1}{N} \sum_{n=1}^N \nabla f_n \varphi_{k,\,n}^{(i)}\right),\\
    \varphi_{k+1,\,n}^{(i)} = & \begin{cases}
x_k^{(i)} & I_k = n,\\
\varphi_{k,\,n}^{(i)} & \text{otherwise},
\end{cases}
\end{align}
for $i = 1, 2$.

For $b>0$, let $V_{b}\left(s,\,s'\right)\triangleq\|x-x'\|_{2}^{2}+b\sum_{n=1}^{N}\|y_{n}-y_{n}'\|_{2}^{2}$. We can choose $\eta\in\left(0,\,c/L^{2}\right)$ and
$b>0$ such that $\eta^{2}<b$ and $\gamma\left(\eta\right)+b\,L^{2}<1$.
We have $\gamma\left(\eta\right)<1$ for $\eta\in\left(0,\,2\,c/L^{2}\right)$.
Next,
$\gamma\left(\eta\right)+\eta^{2}L^{2}=1-2\,\eta\,c+2\,\eta^{2}L^{2}<1$ holds for $\eta\in\left(0,\,c/L^{2}\right)$.

\begin{lem}
\label{lem:SAGA-2}For all $k\geq0$,
\begin{equation*}
 \mathbb{E}\left[V_{b}\left(s_{k+1}^{\left(1\right)},\,s_{k+1}^{\left(2\right)}\right)\,\vert\,\mathcal{F}_{k}\right]
\leq \max\left\{ \gamma\left(\eta\right)+b\,L^{2},\,\frac{\eta^{2}/b+N-1}{N}\right\} V_{b}\left(s_{k}^{\left(1\right)},\,s_{k}^{\left(2\right)}\right).
\end{equation*}
\end{lem}

\begin{proof}
First, in the general iteration
\begin{equation}\label{eq:gradient-inexact}
    x_{k+1}^{(i)} = x_{k}^{(i)} - \eta [\nabla f(x_{k}^{(i)}) + \varepsilon_k^{(i)}],\,\forall k \geq 0,
\end{equation}
for $i = 1, 2$, in the unbiased case where $\mathbb{E}[\varepsilon_k^{(i)} \, \vert \, \mathcal{F}_k] = 0$ we have
\begin{equation}\label{eq:gradient-inexact-1}
    \mathbb{E}[\|x_{k+1}^{(1)} - x_{k+1}^{(2)}\|_2^2 \leq  \gamma(\eta) \|x_{k}^{(1)} - x_{k}^{(2)}\|_2^2 + \eta^2 \mathbb{E}[\|\varepsilon_k^{(1)} - \varepsilon_k^{(2)}\|_2^2],
\end{equation}
using $\mathbb{E}[\varepsilon_k^{i} \, \vert \, \mathcal{F}_k] = 0$ for $i = 1,\,2$. Then, for SAGA we have
\begin{align*}
 & \mathbb{E}\left[\|\Delta x_{k+1}\|_{2}^{2}\,\vert\,\mathcal{F}_{k}\right]\\
\leq\, & \gamma\left(\eta\right)\|\Delta x_{k}\|_{2}^{2}+\eta^{2}\mathbb{E}\Bigg[\|\nabla f_{I_k} (\varphi_{k,\,I_k}^{\left(1\right)})-\nabla f_{I_k} (\varphi_{k,\,I_k}^{\left(2\right)})\\
& \hspace{13em}-\frac{1}{N}\sum_{n=1}^{N}\left(\nabla f_{n} (\varphi_{k,\,n}^{\left(1\right)})-\nabla f_{n} (\varphi_{k,\,n}^{\left(2\right)})\right)\|_{2}^{2}\,\vert\,\mathcal{F}_{k}\Bigg]\\
\leq\, & \gamma\left(\eta\right)\|\Delta x_{k}\|_{2}^{2}+\eta^{2}\mathbb{E}\left[\|\nabla f_{I_k} (\varphi_{k,\,I_k}^{\left(1\right)})-\nabla f_{I_k} (\varphi_{k,\,I_k}^{\left(2\right)})\|_{2}^{2}\,\vert\,\mathcal{F}_{k}\right],
\end{align*}
where the first inequality is by Eq. \eqref{eq:gradient-inexact-1} and the second
follows because
\[
\mathbb{E}\left[\nabla f_{I_k} (\varphi_{k,\,I_k}^{\left(1\right)})-\nabla f_{I_k} (\varphi_{k,\,I_k}^{\left(2\right)})\,\vert\,\mathcal{F}_{k}\right]=\frac{1}{N}\sum_{n=1}^{N}\left(\nabla f_{n} (\varphi_{k,\,n}^{\left(1\right)})-\nabla f_{n} (\varphi_{k,\,n}^{\left(2\right)})\right),
\]
and the variance is less than the second-order moment. Next, we show
\begin{align*}
 & \mathbb{E}\left[\sum_{n=1}^{N}\|\nabla f_{n} (\varphi_{k+1,\,n}^{\left(1\right)})-\nabla f_{n} (\varphi_{k+1,\,n}^{\left(2\right)})\|_{2}^{2}\right]\\
\leq\, & \frac{1}{N}\sum_{n=1}^{N}\|B_{n}\left(x_{k}^{\left(1\right)}\right)-B_{n}\left(x_{k}^{\left(2\right)}\right)\|_{2}^{2}+\frac{N-1}{N}\sum_{n=1}^{N}\|\nabla f_{n} (\varphi_{k,\,n}^{\left(1\right)})-\nabla f_{n} (\varphi_{k,\,n}^{\left(2\right)})\|_{2}^{2}\\
\leq\, & L^{2}\|\Delta x_{k}\|_{2}^{2}+\frac{N-1}{N}\sum_{n=1}^{N}\|\nabla f_{n} (\varphi_{k,\,n}^{\left(1\right)})-\nabla f_{n} (\varphi_{k,\,n}^{\left(2\right)})\|_{2}^{2}.
\end{align*}
Now to conclude, we take
\begin{align*}
 & \mathbb{E}\left[V_{b}\left(s_{k+1}^{\left(1\right)},\,s_{k+1}^{\left(2\right)}\right)\,\vert\,\mathcal{F}_{k}\right]\\
=\, & \mathbb{E}\left[\|\Delta x_{k+1}\|_{2}^{2}+b\sum_{n=1}^{N}\|\nabla f_{n} (\varphi_{k+1,\,n}^{\left(1\right)})-\nabla f_{n} (\varphi_{k+1,\,n}^{\left(2\right)})\|_{2}^{2}\,\vert\,\mathcal{F}_{k}\right]\\
\leq\, & \gamma\left(\eta\right)\|\Delta x_{k}\|_{2}^{2}+\frac{\eta^{2}}{N}\sum_{n=1}^{N}\|\nabla f_{n} (\varphi_{k,\,n}^{\left(1\right)})-\nabla f_{n} (\varphi_{k,\,n}^{\left(2\right)})\|_{2}^{2}\\
 & +b\left[L^{2}\|\Delta x_{k}\|_{2}^{2}+\frac{N-1}{N}\sum_{n=1}^{N}\|\nabla f_{n} (\varphi_{k,\,n}^{\left(1\right)})-\nabla f_{n} (\varphi_{k,\,n}^{\left(2\right)})\|_{2}^{2}\right]\\
=\, & \left(\gamma\left(\eta\right)+b\,L^{2}\right)\|\Delta x_{k}\|_{2}^{2}+\left(\frac{\eta^{2}+b\left(N-1\right)}{N}\right)\sum_{n=1}^{N}\|\nabla f_{n} (\varphi_{k,\,n}^{\left(1\right)})-\nabla f_{n} (\varphi_{k,\,n}^{\left(2\right)})\|_{2}^{2}\\
\leq\, & \max\left\{ \gamma\left(\eta\right)+b\,L^{2},\,\frac{\eta^{2}/b+N-1}{N}\right\} V_{b}\left(s_{k}^{\left(1\right)},\,s_{k}^{\left(2\right)}\right).
\end{align*}
\end{proof}

\section*{Appendix B: Supplement for Section \ref{sec:examples-epoch} (epoch-based optimization algorithms)}

\subsection*{Details for HSAG}

We can choose the step-size
\[
\eta\in\left(0,\,\frac{2\,c}{\left(1+|S|/N\right)L^{2}}\right)
\]
and the parameter $b>0$ such that: (i) $\eta^{2}<b$; (ii) $\gamma\left(\eta\right)+b\,|S|\,L^{2}/N<1$; and (iii)
$\gamma\left(\eta\right)+n^{2}|S|\,L^{2}/N=1-2\,\eta\,c+\left(1+|S|/N\right)\eta^{2}L^{2}<1$
since $\left(1+|S|/N\right)\eta^{2}L^{2}<2\,\eta\,c$. We also introduce the constant
\[
K\left(\eta\right)\triangleq\max\left\{ \gamma\left(\eta\right)+b\,|S|\,L^{2}/N,\,\frac{\eta^{2}/b+N-1}{N}\right\} .
\]
Within epoch $k \geq 0$, we let $\{\tilde{\mathcal F}_m\}_{m \geq 0}$ be the $\sigma$-algebra generated by $\{\tilde{s}_m\}_{m \geq 0}$.

\begin{lem}
\label{lem:HSAG-1}(i) For fixed $k \geq 0$, for all $m\geq0$,
\begin{equation*}
 \mathbb{E}\left[V_{b,\,S}\left(\tilde{s}_{m+1}^{\left(1\right)},\,\tilde{s}_{m+1}^{\left(2\right)}\right)\,\vert\,\tilde{\mathcal{F}}_{m}\right] \leq K\left(\eta\right)V_{b,\,S}\left(\tilde{s}_{m}^{\left(1\right)},\,\tilde{s}_{m}^{\left(2\right)}\right)+\frac{\eta^{2}L^{2}|S^{C}|}{N}V_{b,\,S}\left(s_{k}^{\left(1\right)},\,s_{k}^{\left(2\right)}\right).
\end{equation*}
(ii) For each epoch $k\geq0$, we have
\begin{equation*}
 \mathbb{E}\left[V_{b,\,S}\left(s_{k+1}^{\left(1\right)},\,s_{k+1}^{\left(2\right)}\right)\,\vert\,\mathcal{F}_{k}\right]\leq \left[K\left(\eta\right)^{M}+\frac{\eta^{2}L^{2}|S^{C}|}{N\left(1-K\left(\eta\right)\right)}\left(1-K\left(\eta\right)^{M}\right)\right]V_{b,\,S}\left(s_{k}^{\left(1\right)},\,s_{k}^{\left(2\right)}\right).
\end{equation*}
\end{lem}

\begin{proof}
(i) First, we have
\begin{align*}
 & \mathbb{E}\left[\|\tilde{x}_{m+1}^{\left(1\right)}-\tilde{x}_{m+1}^{\left(2\right)}\|_{2}^{2}\,\vert\,\mathcal{F}_{k}\right]\\
\leq\, & \gamma\left(\eta\right)\|\tilde{x}_{m}^{\left(1\right)}-\tilde{x}_{m}^{\left(2\right)}\|_{2}^{2}+\eta^{2}\mathbb{E}\Bigg[\|\nabla f_{I_m}(\varphi_{m,\,I_m}^{\left(1\right)})-\nabla f_{I_m}(\varphi_{m,\,I_m}^{\left(2\right)})\\
& \hspace{13em} -\frac{1}{N}\sum_{n=1}^{N}\left(\nabla f_{n}(\varphi_{m,\,n}^{\left(1\right)})-\nabla f_{n}(\varphi_{m,\,n}^{\left(2\right)})\right)\|_{2}^{2}\,\vert\,\tilde{\mathcal{F}}_{m}\Bigg]\\
\leq\, & \gamma\left(\eta\right)\|\tilde{x}_{m}^{\left(1\right)}-\tilde{x}_{m}^{\left(2\right)}\|_{2}^{2}+\eta^{2}\mathbb{E}\left[\|\nabla f_{I_m}(\varphi_{m,\,I_m}^{\left(1\right)})-\nabla f_{I_m}(\varphi_{m,\,I_m}^{\left(2\right)})\|_{2}^{2}\,\vert\,\tilde{\mathcal{F}}_{m}\right]\\
\leq\, & \gamma\left(\eta\right)\|\tilde{x}_{m}^{\left(1\right)}-\tilde{x}_{m}^{\left(2\right)}\|_{2}^{2}+\frac{\eta^{2}}{N}\sum_{n\in S}\|\nabla f_{n}(\varphi_{m,\,n}^{\left(1\right)})-\nabla f_{n}(\varphi_{m,\,n}^{\left(2\right)})\|_{2}^{2}+\frac{\eta^{2}L^{2}|S^{C}|}{N}\|x_{k}^{\left(1\right)}-x_{k}^{\left(2\right)}\|_{2}^{2},
\end{align*}
where the first inequality is by Eq. \eqref{eq:gradient-inexact-1} and the second
follows because
\[
\mathbb{E}\left[\nabla f_{I_k} (\varphi_{k,\,I_k}^{\left(1\right)})-\nabla f_{I_k} (\varphi_{k,\,I_k}^{\left(2\right)})\,\vert\,\mathcal{F}_{k}\right]=\frac{1}{N}\sum_{n=1}^{N}\left(\nabla f_{n} (\varphi_{k,\,n}^{\left(1\right)})-\nabla f_{n} (\varphi_{k,\,n}^{\left(2\right)})\right),
\]
and the variance is less than the second-order moment. Next, we have
\begin{align*}
 & \mathbb{E}\left[\sum_{n\in S}\|\nabla f_{n}(\varphi_{m+1,\,n}^{\left(1\right)})-\nabla f_{n}(\varphi_{m+1,\,n}^{\left(2\right)})\|_{2}^{2}\right]\\
\leq\, & \frac{1}{N}\sum_{n\in S}\|\nabla f_{n}\left(\tilde{x}_{m}^{\left(1\right)}\right)-\nabla f_{n}\left(\tilde{x}_{m}^{\left(2\right)}\right)\|_{2}^{2}+\frac{N-1}{N}\sum_{n\in S}\|\nabla f_{n}(\varphi_{m,\,n}^{\left(1\right)})-\nabla f_{n}(\varphi_{m,\,n}^{\left(2\right)})\|_{2}^{2}\\
\leq\, & \frac{L^{2}|S|}{N}\|\tilde{x}_{m}^{\left(1\right)}-\tilde{x}_{m}^{\left(2\right)}\|_{2}^{2}+\frac{N-1}{N}\sum_{n\in S}\|\nabla f_{n}(\varphi_{m,\,n}^{\left(1\right)})-\nabla f_{n}(\varphi_{m,\,n}^{\left(2\right)})\|_{2}^{2}.
\end{align*}
Now to conclude, we take
\begin{align*}
 & \mathbb{E}\left[V_{b,\,S}\left(\tilde{s}_{m+1}^{\left(1\right)},\,\tilde{s}_{m+1}^{\left(2\right)}\right)\,\vert\,\tilde{\mathcal{F}}_{m}\right]\\
=\, & \mathbb{E}\left[\|\tilde{x}_{m+1}^{\left(1\right)}-\tilde{x}_{m+1}^{\left(2\right)}\|_{2}^{2}+b\sum_{n\in S}\|\nabla f_{n}(\varphi_{m+1,\,n}^{\left(1\right)})-\nabla f_{n}(\varphi_{m+1,\,n}^{\left(2\right)})\|_{2}^{2}\,\vert\,\tilde{\mathcal{F}}_{m}\right]\\
\leq\, & \gamma\left(\eta\right)\|\tilde{x}_{m}^{\left(1\right)}-\tilde{x}_{m}^{\left(2\right)}\|_{2}^{2}+\frac{\eta^{2}}{N}\sum_{n\in S}\|\nabla f_{n}(\varphi_{m,\,n}^{\left(1\right)})-\nabla f_{n}(\varphi_{m,\,n}^{\left(2\right)})\|_{2}^{2}+\frac{\eta^{2}L^{2}|S^{C}|}{N}\|x_{k}^{\left(1\right)}-x_{k}^{\left(2\right)}\|_{2}^{2}\\
 & +b\left[\frac{L^{2}|S|}{N}\|\tilde{x}_{m}^{\left(1\right)}-\tilde{x}_{m}^{\left(2\right)}\|_{2}^{2}+\frac{N-1}{N}\sum_{n\in S}\|\nabla f_{n}(\varphi_{m,\,n}^{\left(1\right)})-\nabla f_{n}(\varphi_{m,\,n}^{\left(2\right)})\|_{2}^{2}\right]\\
=\, & \left(\gamma\left(\eta\right)+b\frac{L^{2}|S|}{N}\right)\|\tilde{x}_{m}^{\left(1\right)}-\tilde{x}_{m}^{\left(2\right)}\|_{2}^{2}+\left(\frac{\eta^{2}+b\left(N-1\right)}{N}\right)\sum_{n\in S}\|\nabla f_{n}(\varphi_{m,\,n}^{\left(1\right)})-\nabla f_{n}(\varphi_{m,\,n}^{\left(2\right)})\|_{2}^{2}\\
 & +\frac{\eta^{2}L^{2}|S^{C}|}{N}\|x_{k}^{\left(1\right)}-x_{k}^{\left(2\right)}\|_{2}^{2}\\
\leq\, & \max\left\{ \gamma\left(\eta\right)+b\frac{L^{2}|S|}{N},\,\frac{\eta^{2}/b+N-1}{N}\right\} V_{b,\,S}\left(\tilde{s}_{m}^{\left(1\right)},\,\tilde{s}_{m}^{\left(2\right)}\right)\\
 & +\frac{\eta^{2}L^{2}|S^{C}|}{N}\left(\|x_{k}^{\left(1\right)}-x_{k}^{\left(2\right)}\|_{2}^{2}+b\sum_{n\in S}\|\nabla f_{n} (\varphi_{k,\,n}^{\left(1\right)})-\nabla f_{n} (\varphi_{k,\,n}^{\left(2\right)})\|_{2}^{2}\right).
\end{align*}
Then we conclude
\begin{equation*}
 \mathbb{E}\left[V_{b,\,S}\left(\tilde{s}_{m+1}^{\left(1\right)},\,\tilde{s}_{m+1}^{\left(2\right)}\right)\,\vert\,\tilde{\mathcal{F}}_{m}\right] \leq K\left(\eta\right)V_{b,\,S}\left(\tilde{s}_{m}^{\left(1\right)},\,\tilde{s}_{m}^{\left(2\right)}\right)+\frac{\eta^{2}L^{2}|S^{C}|}{N}V_{b,\,S}\left(s_{k}^{\left(1\right)},\,s_{k}^{\left(2\right)}\right).
\end{equation*}

(ii) Follows by iterating the above recursion and using $\tilde{s}_{0}^{\left(i\right)}=s_{k}^{\left(i\right)}$ and $\tilde{\varphi}_{0,n}^{(i)} = \nabla f_n(x_k^{(i)})$ for all $i = 1, 2$.
\end{proof}

\subsection*{Details for Catalyst}

Here we confirm that Catalyst has an invariant distribution, and we compute the rate of convergence in the Wasserstein divergence. First recall that by \cite[Proposition 5]{lin2017catalyst} we have
$$
\mathbb E[\psi(x_k) - \psi(x^*)] \leq \frac{8}{(\sqrt q - \alpha)^2}(1 - \alpha)^{k+1}(\psi(x_0) - \psi(x^*)),\,\forall k \geq 0.
$$

Let $\{\mathfrak Q_k\}_{k \geq 0}$ denote the sequence of transition kernels corresponding to $\{\hat{T}_k\}_{k \geq 0}$ for Catalyst. These transition kernels are time varying, so we no longer have $\mu_k = \mu\, \mathfrak Q^k$ for initial $\mu \in \mathcal{P}_{\bar{V}}(\ALP S)$ as for i.i.d. $\{\hat{T}_k\}_{k \geq 0}$.

Choose any $\mu \in \mathcal{P}_{\bar{V}}(\ALP S)$, define $\mu_{k+1} = \mu_{k} \mathfrak Q_k$ for all $k \geq 0$, and let $s_k$ be distributed according to $\mu_k$ for all $k \geq 0$. We first confirm that $\{\mu_k\}_{k \geq 0}$ is a Cauchy sequence in the Wasserstein divergence. Pick any $l \in \Na$, then we have
\begin{align}
& W_{\bar{V}}\left(\mu_k,\,\mu_{k+l}\right)\\
\leq\, & \mathbb{E}[\bar{V}(s_k,\,s_{k+l})]\\
\leq\, & \mathbb{E}[\psi(x_k) + \psi(x_{k+l}) - 2\,\psi^* + (1 - \alpha)(\psi(x_{k-1}) + \psi(x_{k+l-1}) - 2\,\psi^*)]\\
\leq\, & \frac{16}{(\sqrt q - \alpha)^2}(\psi(x_0) - \psi(x^*))[(1 - \alpha)^{k+1} + (1 - \alpha)^{k+l+1}]\\
=\, & \frac{16}{(\sqrt q - \alpha)^2}(\psi(x_0) - \psi(x^*))(1 - \alpha)^{k+1}[1 + (1 - \alpha)^{l}],
\end{align}
where we use the coupling $s_{k+l} = \hat{T}_{k+l-1}\circ \hat{T}_{k+l-2} \circ \dots \circ \hat{T}_{k} s_k$. This display establishes that $\{\mu_k\}_{k \geq 0}$ is Cauchy, and by Proposition \ref{prop:WVplimitexistence} there exists a probability measure $\vartheta$ such that $\lim_{k\rightarrow\infty} W_{\bar{V}}\left(\mu_k,\,\vartheta\right)=0$.

Next we show that this limit is the same for all initial $\mu \in \mathcal{P}_{\bar{V}}(\ALP S)$. Let $\mu,\,\nu \in \mathcal{P}_{\bar{V}}(\ALP S)$ and define $\mu_{k+1} = \mu_k \mathfrak Q_k$ where $\mu_0 = \mu$ and $\nu_{k+1} = \nu_k \mathfrak Q_k$ where $\nu_0 = \nu$. Then $\mu_k$ converges to some $\vartheta_{\mu} \in \mathcal{P}_{\bar{V}}(\ALP S)$ and $\nu_k$ converges to some $\vartheta_{\nu} \in \mathcal{P}_{\bar{V}}(\ALP S)$. In addition, if we let $s_k^{\mu}$ be distributed according to $\mu_k$ and $s_k^{\nu}$ be distributed according to $\nu_k$ for all $k \geq 0$, then
\begin{align}
W_{\bar{V}}\left(\mu_k,\,\nu_k\right) \leq\, & \mathbb{E}[\bar{V}(s_k^{\mu},\,s_k^{\nu})]\\
\leq\, & \mathbb{E}[\psi(x_k^{\mu}) + \psi(x_{k}^{\nu}) - 2\,\psi^* + (1 - \alpha)(\psi(x_{k-1}^{\mu}) + \psi(x_{k-1}^{\nu}) - 2\,\psi^*)]\\
\leq\, & \frac{16}{(\sqrt q - \alpha)^2}\mathbb{E}[\psi(x_0^{\mu}) + \psi(x_0^{\nu}) - 2\,\psi(x^*)](1 - \alpha)^{k+1},
\end{align}
so it follows that $\vartheta_{\mu} = \vartheta_{\nu} = \vartheta$.

Finally, we compute $\vartheta$ explicitly. Let $s^* = (x^*,\,x^*)$ for Catalyst, then $\vartheta = \ind{s^*}$ is invariant for all $\mathfrak Q_k$. In particular, $L_k s^* = x^*$ for all $k \geq 0$, and so the auxiliary optimization problems
$$
\psi_k(x;\,s^*) \triangleq f(x) + g(x) + \frac{\theta}{2} \|x - x^*\|_2^2
$$
are the same for all $k \geq 0$, with unique optimal solution $x^*$.

\end{document}